%% file: main.tex
\documentclass[10pt,twocolumn,letterpaper]{article}

\input{preamble}

\def\cvprPaperID{847} 

\def\confYear{2023}

\begin{document}

\title{Neural Kernel Surface Reconstruction}

\author{\vspace{-1.5em}Jiahui Huang$^1$ \;\; Zan Gojcic$^1$ \;\; Matan Atzmon$^1$ \;\; Or Litany$^1$ \;\; Sanja Fidler$^{1,2,3}$ \;\; Francis Williams$^1$\\
\small\normalfont $^1$NVIDIA \quad $^2$University of Toronto \quad $^3$Vector Institute
}


\input{figures/teaser}

\input{sections/sec0-abstract.tex}

\input{sections/sec1-intro.tex}
\input{sections/sec2-related.tex}
\input{sections/sec3-method.tex}
\input{sections/sec4-exp.tex}
\input{sections/sec5-conclusion.tex}

{\small
\bibliographystyle{ieee_fullname}
\bibliography{egbib}
}

\clearpage
\section*{Appendix}
\appendix
\input{sections/sup1-method}
\input{sections/sup2-setting}
\input{sections/sup3-ext}
\input{sections/sup4-vis}
\end{document}

%% file: preamble.tex
\usepackage{cvpr}              
\cvprfinalcopy 

\usepackage[utf8]{inputenc}
\usepackage{times}
\usepackage{epsfig}
\usepackage{graphicx}
\usepackage{amsmath,mathtools}
\usepackage{amssymb}
\usepackage{amsthm}
\usepackage{nicefrac}       
\usepackage{microtype}      
\usepackage{amsmath} 
\usepackage{float}
\usepackage{amsfonts}
\usepackage{bm}
\usepackage{enumitem}
\usepackage{mathtools}
\usepackage{multirow}

\usepackage{fontawesome}
\usepackage{pifont}
\usepackage{color}
\usepackage{xcolor}
\usepackage[algo2e,inoutnumbered,linesnumbered,algoruled,vlined]{algorithm2e}
\usepackage{algpseudocode}
\usepackage{booktabs}
\usepackage{tabularx}
\usepackage{bm}
\usepackage{url}
\usepackage{bbm}
\usepackage{threeparttable}
\usepackage{tikz}
\usepackage{subfig}
\usepackage{wrapfig}
\usetikzlibrary{backgrounds}
\usetikzlibrary{decorations.pathreplacing}
\usepackage[outline]{contour}
\usepackage[pagebackref=true,breaklinks=true,colorlinks,bookmarks=false]{hyperref}
\usepackage[affil-it]{authblk}
\usepackage[accsupp]{axessibility}  

\newcommand{\parahead}[1]{\vspace{1.5mm}\noindent\textbf{#1.}\ }

\DeclareMathOperator*{\argmin}{argmin}

\newcommand{\MethodName}{NKSR\xspace}

\newtheorem{theorem}{Theorem}

\newtheorem{lemma}[theorem]{Lemma}

\newtheorem{remark}[theorem]{Remark}

\theoremstyle{definition}

\newcommand{\RR}{\mathbb{R}}

\newcommand{\Xin}{\bm{X}_\text{in}}
\newcommand{\Oin}{\bm{O}_\text{in}}
\newcommand{\Nin}{\bm{N}_\text{in}}
\newcommand{\Xdense}{\bm{X}_\text{dense}}
\newcommand{\Odense}{\bm{O}_\text{dense}}
\newcommand{\Ndense}{\bm{N}_\text{dense}}

\newcommand{\Seps}{\bm{S}_\epsilon}
\newcommand{\Qmat}{\mathbf{Q}}
\newcommand{\Gmat}{\mathbf{G}}
\newcommand{\x}{\bm{x}}
\newcommand{\trans}{\mathbf{T}}
\newcommand{\Xgt}{\bm{X}_\text{gt}}
\newcommand{\xgt}{\bm{x}_\text{gt}}
\newcommand{\Ngt}{\bm{N}_\text{gt}}
\newcommand{\ngt}{\bm{n}_\text{gt}}
\newcommand{\Xpd}{\bm{X}_\text{pd}}
\newcommand{\xpd}{\bm{x}_\text{pd}}
\newcommand{\Npd}{\bm{N}_\text{pd}}
\newcommand{\npd}{\bm{n}_\text{pd}}

\DeclareRobustCommand{\ks}{%
\begingroup\normalfont
\includegraphics[height=1.5\fontcharht\font`\B]{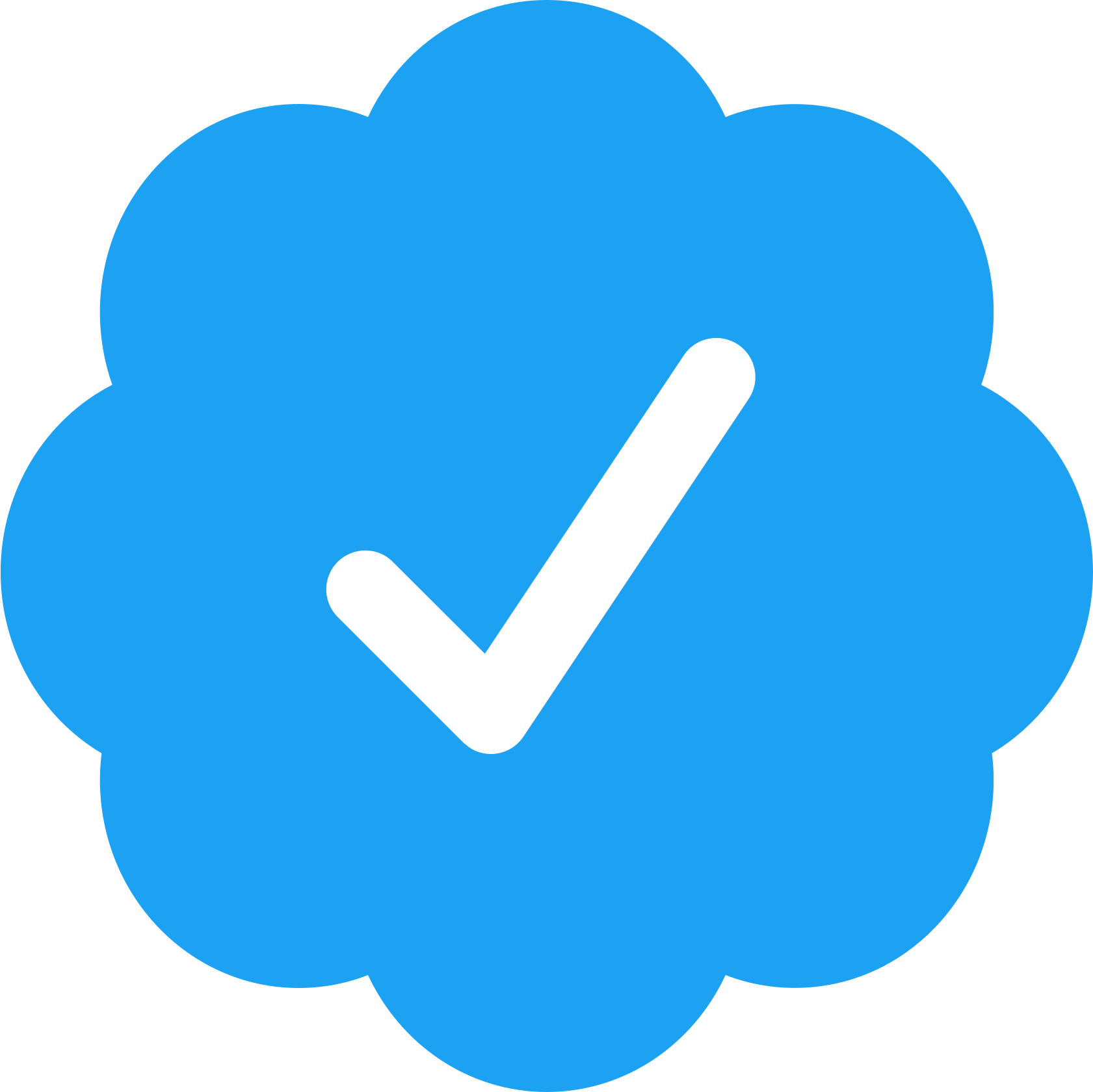}%
\endgroup}

\definecolor{darkblue}{RGB}{49,130,189}
\definecolor{stanfordgrey}{RGB}{46,45,41}
\definecolor{cardinalred}{RGB}{253,141,60}

\usepackage[capitalize]{cleveref}

\crefname{section}{\S}{\S\S}
\crefname{subsection}{\S}{\S\S}
\crefname{conj}{Conj.}{Conj.}

\Crefname{assumption}{\textbf{H}\hspace{-3pt}}{\textbf{H}\hspace{-3pt}}
\crefname{assumption}{\textbf{H}}{\textbf{H}}

\crefname{algorithm}{\text{Alg.}}{\text{Alg.}}
\crefname{assumption}{\textbf{H}}{\textbf{H}}
\crefname{equation}{\text{Eq}}{\text{Eq}}
\crefname{definition}{\text{Dfn.}}{\text{Dfn.}}
\crefname{lemma}{\text{Lemma}}{\text{Lemma}}
\crefname{dfn}{\text{Dfn.}}{\text{Dfn.}}
\crefname{thm}{\text{Thm.}}{\text{Thm.}}
\crefname{tab}{\text{Tab.}}{\text{Tab.}}
\crefname{fig}{\text{Fig.}}{\text{Fig.}}
\crefname{table}{\text{Tab.}}{\text{Tab.}}
\crefname{figure}{\text{Fig.}}{\text{Fig.}}


%% file: figures/teaser.tex
\twocolumn[{%
\renewcommand\twocolumn[1][]{#1}%
\vspace{-7.5mm}
\maketitle
\begin{center}
    \vspace{-10mm}
    \small\emph{Project page:} \url{https://research.nvidia.com/labs/toronto-ai/NKSR/}
\end{center}
\begin{center}
    \vspace{-2mm}
    \includegraphics[width=\linewidth]{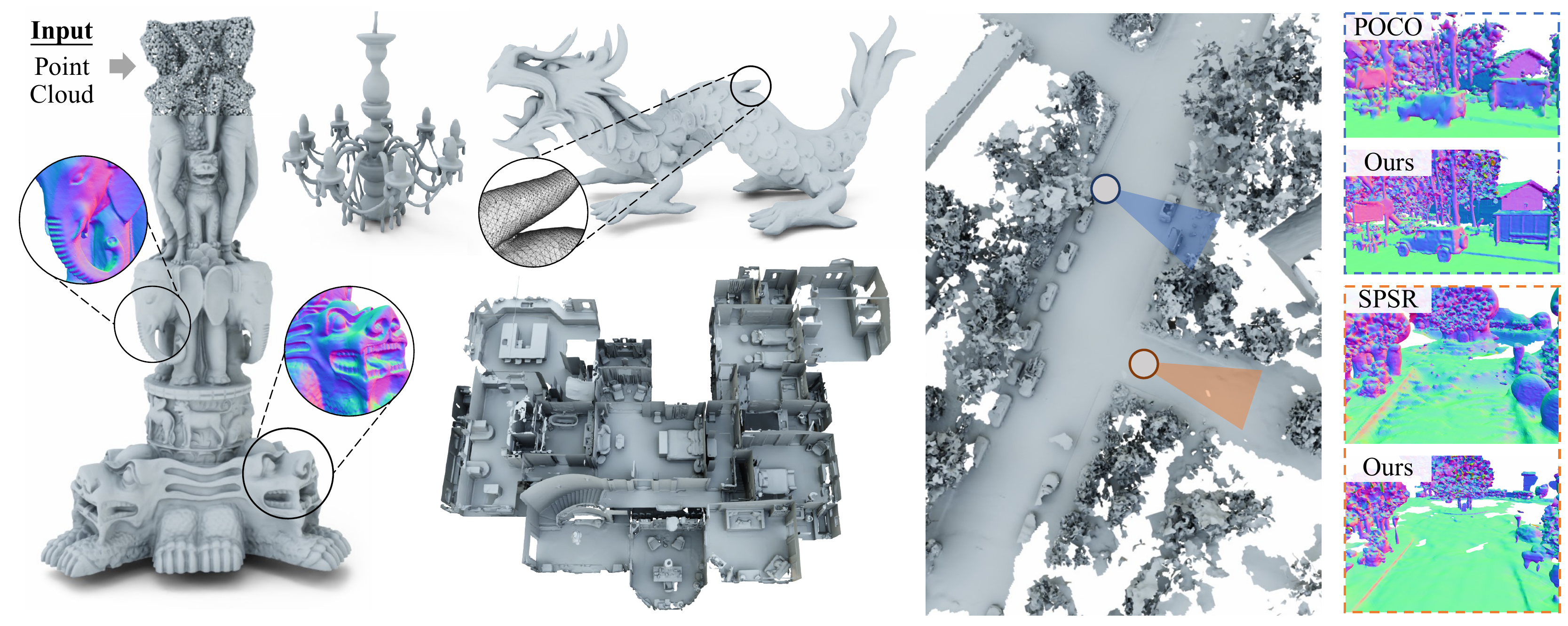}
    \vspace{-7.5mm}
    \captionof{figure}{We present Neural Kernel Surface Reconstruction (\textbf{\MethodName{}}) for recovering a 3D surface from an input point cloud. Trained directly from dense points, our method reaches state-of-the-art reconstruction quality and scalability. \MethodName{} is also highly generalizable: All the meshes in this figure are reconstructed using a single trained model.}
    \label{fig:teaser}
\end{center}%
}]

%% file: sections/sec0-abstract.tex
\begin{abstract}
\vspace{-1em}
We present a novel method for reconstructing a 3D implicit surface from a large-scale, sparse, and noisy point cloud. 
Our approach builds upon the recently introduced Neural Kernel Fields (NKF)~\cite{williams2022neural} representation. It enjoys similar generalization capabilities to NKF, while simultaneously addressing its main limitations: (a) We can scale to large scenes through compactly supported kernel functions, which enable the use of memory-efficient sparse linear solvers. (b) We are robust to noise, through a gradient fitting solve. (c) We minimize training requirements, enabling us to learn from any dataset of dense oriented points, and even mix training data consisting of objects and scenes at different scales.
Our method is capable of reconstructing millions of points in a few seconds, and handling very large scenes in an out-of-core fashion. We achieve state-of-the-art results on reconstruction benchmarks consisting of single objects (ShapeNet~\cite{chang2015shapenet}, ABC~\cite{Koch_2019_CVPR}), indoor scenes (ScanNet~\cite{dai2017scannet}, Matterport3D~\cite{Matterport3D}), and outdoor scenes (CARLA~\cite{Dosovitskiy17}, Waymo~\cite{Sun_2020_CVPR}).

\end{abstract}

%% file: sections/sec1-intro.tex
\vspace{-5mm}
\section{Introduction}
\label{sec:introduction}

The goal of 3D reconstruction is to recover geometry from partial measurements of a shape. In this work, we aim to map a sparse set of oriented points sampled from the surface of a shape to a 3D implicit field.
This is a challenging inverse problem since point clouds acquired from real-world sensors are often very large (millions or billions of points), vary in sampling density, and are corrupted with sensor noise. Furthermore, since surfaces are continuous but points are discrete, there are many valid solutions which can explain a given input.
%
%
%
To address these issues, past approaches aim to recover surfaces 
that agree with the input points 
while satisfying some prior everywhere else in space. Classical methods use 
an explicit prior (e.g. smoothness), while more recent learning-based approaches promote a likely reconstruction under a data-driven prior. 

There are, however, key limitations to both types of techniques that inhibit their application in practical situations. 
%
Since classical methods are fast, scalable, and able to handle diverse inputs, they have become an industry standard (e.g. \cite{kazhdan2013screened, williams2020nsplines}). Yet, they
suffer from quality issues in the presence of high noise or sparse inputs,
often failing to reconstruct even simple geometry such as a ground plane (see the ground in \cref{fig:teaser}).
On the other hand, learning-based approaches were shown to handle large noise \cite{peng2021shape}, and sparse inputs \cite{pan2019deep, boulch2022poco}, yet they often
struggle to generalize to out-of-distribution shapes and sampling densities as was highlighted in \cite{williams2022neural}.
These generalization issues can be attributed to the fact that current learning-based methods struggle to exploit large and diverse amounts of data for training. One cause of this is that a single forward pass can take minutes for even moderately sized inputs (\textit{e.g.} \cite{boulch2022poco}), limiting training to collections consisting of small point clouds. Furthermore, many existing methods rely on a preprocessing step to extract supervision in the form of occupancy or signed distance function \cite{peng2020convoccnet, mescheder2019occnet, park2019deepsdf, boulch2022poco, williams2022neural}. 
In practice, this preprocessing step hinders the ability to easily use diverse datasets for training since most shape datasets (including synthetic ones such as the popular ShapeNet \cite{chang2015shapenet}) consist of non-watertight shapes, open surfaces, or contain ghost geometry from which extracting supervision is hard. 

Recently, \cite{williams2022neural} proposed Neural Kernel Fields (NKF), a new paradigm to address the problem of generalization in 3D 
reconstruction. NKF learns a data-dependent kernel, and predicts a continuous occupancy field as a linear combination of this kernel supported on the input points.  %
The key insights of NKF are that a kernel explicitly encodes inductive bias, and that solving a kernel linear interpolation problem at test time always produces solutions that adhere to the inputs. Thus, by training on diverse shapes, NKF can learn a good inductive bias for the general 3D reconstruction problem rather than for a specific dataset. While NKF achieves impressive generalization results, it suffers from two major limitations that restrict its practical application. First, since it uses a globally supported kernel, it requires solving a dense linear system and cannot reconstruct inputs with more than ten thousand input points. Second, it degrades poorly in the presence of noise due to its interpolation of exact positional occupancy constraints.

In this work, we build upon the excellent generalization capability of NKF and tackle its main limitations to achieve a \emph{practical} learning-based reconstruction method that is scalable, fast, and robust to noise. 
Like NKF, our work leverages the idea of a learned kernel for generalization, but we (1) develop a novel, gradient-based kernel formulation which is robust to noise, and (2) use an explicit voxel hierarchy structure and compactly supported kernels to make our interpolation problem sparse, multi-scale, and capable of handling large inputs while still producing high fidelity outputs. The result is a learning-based yet out-of-the-box reconstruction method that can be applied to point clouds in the wild.
In particular, it enjoys all of the following properties: 
\begin{itemize}[topsep=1pt,itemsep=0pt]
\item It can generalize to out-of-distribution inputs, producing high-fidelity reconstructions, even in the presence of sparsity and noise.
\item It can be trained on the union of diverse datasets while only requiring dense oriented points as supervision, unlocking a new level of training data scale.
\item It can reconstruct point clouds consisting of millions of points in seconds, and scale to extremely large inputs in an out-of-core fashion.
\end{itemize}
We illustrate other methods in the context of these points visually in \cref{fig:venn}.

\input{figures/venn}

\input{figures/architecture}

%% file: figures/venn.tex
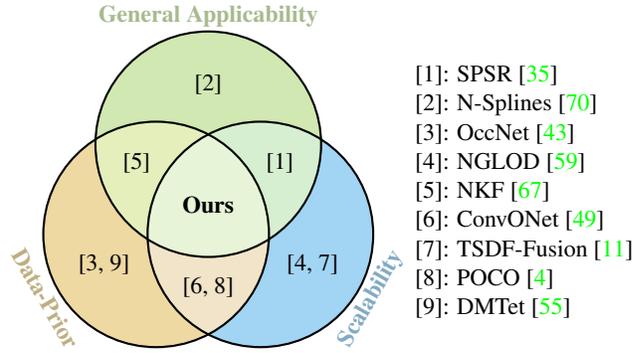
\begin{figure}[!t]
    \definecolor{color_1}{RGB}{209, 232, 178}
    \definecolor{color_2}{RGB}{239, 217, 163}
    \definecolor{color_3}{RGB}{162, 212, 244}

    \centering
    \begin{tikzpicture}
        \small
        \begin{scope}[blend group = soft light]
          \fill[color_1]   ( 90:0.8) circle (1.5);
          \fill[color_2] (210:0.8) circle (1.5);
          \fill[color_3]  (330:0.8) circle (1.5);
        \end{scope}
        \begin{scope}
            \draw[black, thick]   ( 90:0.8) circle (1.5);
            \draw[black, thick] (210:0.8) circle (1.5);
            \draw[black, thick]  (330:0.8) circle (1.5);
        \end{scope}
        \node [color=color_1!80!black,font=\bfseries] at ( 90:2.5)    {General Applicability};
        \node [color=color_2!80!black,font=\bfseries,rotate=-60] at ( 210:2.5)   {Data-Prior};
        \node [color=color_3!80!black,font=\bfseries,rotate=60] at ( 330:2.5)   {Scalability};
        \node [font=\bfseries] at (0:0)      {\small Ours};
        \node at ( 90:1.6)    {[2]};
        \node at ( 210:1.6)    {[3, 9]};
        \node at ( 330:1.6)    {[4, 7]};
        \node at ( 150:1.1)    {[5]};
        \node at ( 30:1.1)    {[1]};
        \node at ( 270:1.1)    {[6, 8]};

        \node [align=left] at (4.2,0) {%
            {[1]: SPSR~\cite{kazhdan2013screened}} \\
            {[2]: N-Splines~\cite{williams2020nsplines}} \\
            {[3]: OccNet~\cite{mescheder2019occnet}} \\
            {[4]: NGLOD~\cite{takikawa2021neural}} \\
            {[5]: NKF~\cite{williams2022neural}} \\
            {[6]: ConvONet~\cite{peng2020convoccnet}} \\
            {[7]: TSDF-Fusion~\cite{curless1996volumetric}} \\
            {[8]: POCO~\cite{boulch2022poco}}\\
            {[9]: DMTet~\cite{dmtet21}}\\
        };
    \end{tikzpicture}
    \vspace{-1.5em}
    \caption{\textbf{Comparison to related works.}}
    \label{fig:venn}
    \vspace{-1.5em}
\end{figure}

%% file: figures/architecture.tex
\begin{figure*}[!t]
\vspace{-2mm}
    \centering
    \includegraphics[width=1.0\textwidth]{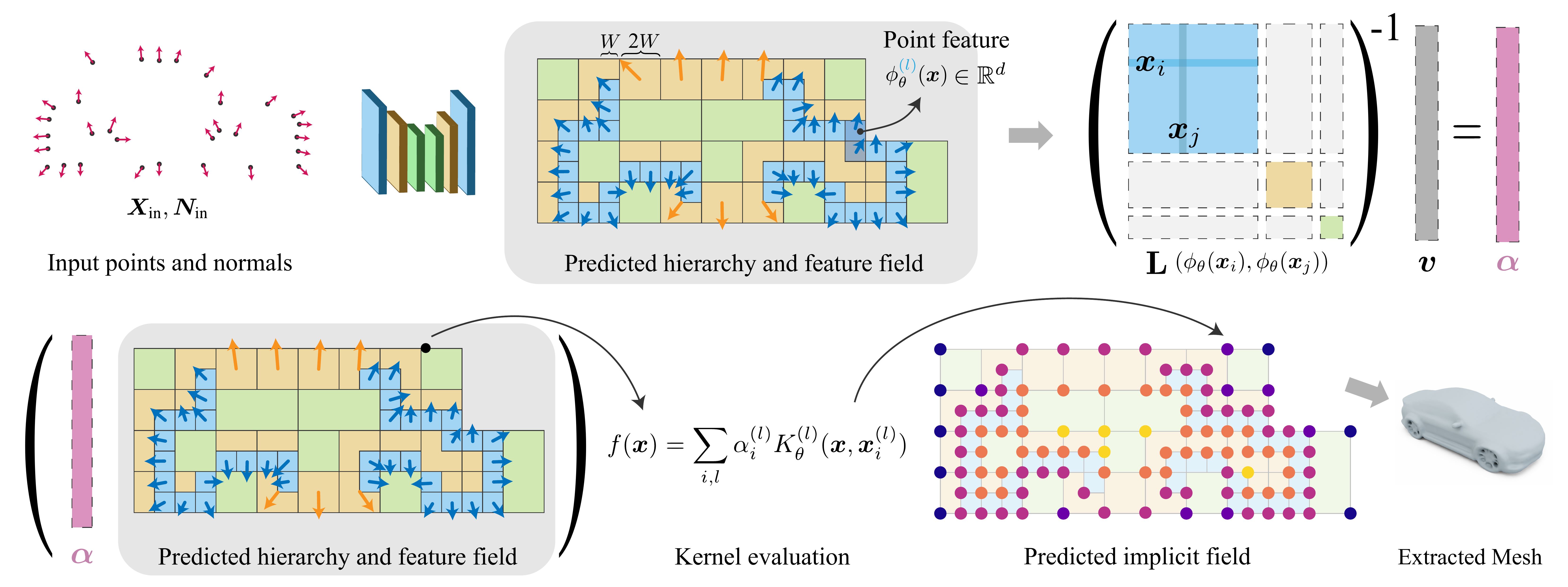}
    \vspace{-5.5mm}
    \caption{\textbf{Pipeline.} Our method accepts an oriented point cloud and predicts a sparse hierarchy of voxel grids containing features as well as normals in each voxel. We then construct a sparse linear system and solve for a set of per-voxel coefficients $\bm{\alpha}$. The linear system corresponds to the gram matrix arising from a kernel which depends on the predicted features, illustrated as $\mathbf{L}$ and $\bm{v}$ above but mathematically defined in \cref{eq:linsys}. To extract the predicted surface, we evaluate the function values at the voxel corners using a linear combination of the learned kernel basis functions, followed by dual marching cubes.}
    \label{fig:arch}
    \vspace{-2.5mm}
\end{figure*}

%% file: sections/sec2-related.tex
\section{Related Work}\label{sec:related}

We now give a brief overview of prior works that are relevant to our approach.
Learned kernels were investigated in~\cite{wilson2015deep, NIPS2007_4b6538a4, patacchiola2020bayesian,guizilini2018learning} for tasks such as few-shot transfer learning, classification of images, or robot mapping.
In the context of 3D reconstruction, Chu \etal\cite{chu2021unsupervised} encoded the inductive bias intrinsically in a 3D CNN structure without training data.
NKF~\cite{williams2022neural} proposed a novel \emph{data-dependent} kernel, which improved upon the non-learned kernel method derived from infinitely wide ReLU networks in Neural Splines~\cite{williams2020nsplines}. 
Comparably, we use a data-dependent kernel but restrict its spatial support to increase computational efficiency and use a gradient-based fitting formulation to increase noise robustness. 
Mapping 3D points to a feature grid via a convolutional architecture was proposed in ConvONet~\cite{peng2020convoccnet} and CIRCLE~\cite{chen2022circle} for predicting an occupancy field.
POCO~\cite{boulch2022poco} improved the quality and performance of ConvONet by using a transformer architecture instead of convolutions. 
Both methods, however, take a long time to reconstruct even a small scene.
Differently, our feature mapping is made efficient through a hierarchical sparse data structure. 
Non-dense data structures were studied in ASR~\cite{Ummenhofer_2021_ICCV} and DOGNN~\cite{wang2022dual} which proposed octree-based convolutional architectures for reconstructing large scenes.  
Generalization to novel scenes was addressed by LIG~\cite{genova2020local,huang2021di} using local patches which have smaller variability, but are very sensitive to the choice of patch size and relies on test-time optimization with unknown convergence properties. 
Similarly, our kernel weights are fitted to the scene at prediction but the fitting is done via a linear solver in a form of meta-learning~\cite{sitzmann2020metasdf}. 
Shape as Points~\cite{peng2021shape} learns to upsample the input points followed by differentiable Poisson reconstruction, and this idea is further extended by NGSolver~\cite{huang2022neuralgalerkin} to incorporate learnable basis functions.
However, the representation power of the surface is still bounded by the chosen family of basis functions where non-trivial integrations have to be applied.
Beyond methods based on implicit surfaces, other shape reconstruction techniques exist which leverage different output representations. 
These representations include dense point clouds \cite{rempe2020caspr, luo2021diffusion, zhou20213d, qi2017pointnet, qi2017pointnetpp, zhao20193d, sun2020canonical, yu2018ecnet, yu2018punet,fan2016point,lin2017learning,ma2023learning}, polygonal meshes \cite{hanocka2020point2mesh, chen2020bspnet, gao2020learning, Hanocka_2019, gkioxari2020mesh, Williams_2020_CVPR_Workshops, deng2020cvxnet,Litany_2018_CVPR,halimi2020towards,dmtet21,ma2023towards}, manifold atlases \cite{Williams_2019, deprelle2019learning, groueix2018atlasnet, gadelha2020deep, badki2020meshlet}, and voxel grids \cite{choy20163dr2n2,tulsiani2020objectcentric,hane2017hierarchical,marrnet,tulsiani2017multiview,girdhar2016learning}. 
While our method uses a neural field for reconstruction, past work has used neural fields to perform a variety of point cloud tasks such as shape compression~\cite{takikawa2021neural, williams2020nsplines}, 
voxel grid upsampling~\cite{peng2020convoccnet, mescheder2019occnet}, reconstruction from rotated inputs~\cite{deng2021vector,atzmon2022frame} and articulated poses~\cite{deng2020nasa,zhang2021strobenet}. 

%% file: sections/sec3-method.tex
\section{Method}
Our method predicts a 3D surface given a point cloud with normals. We encode this predicted surface as the zero level set of a Neural Kernel Field, \ie an implicit function represented as a weighted sum of learned, positive-definite basis functions which are conditioned on the input, and whose weights are computed using a linear optimization in the forward pass. These basis functions are supported on a sparse voxel hierarchy which we predict from the input point cloud using a sparse convolutional network, and depend on interpolated features at each voxel corner. 
In the following sections, we describe the key steps of our method, and summarize it pictorially in \cref{fig:arch}. We additionally provide rigorous derivations for each step in the Appendix.

\subsection{Predicting a 3D Shape from Points}
Given points and normals, the forward pass of our model predicts an implicit surface as a weighted sum of learned kernels in two steps: First, we feed the input to a sparse convolutional network that predicts a voxel hierarchy with features at each corner (\cref{fig:hierarchy_illustration}). These features define a collection of learned basis functions, which are centered at each voxel in the hierarchy. Second, we find a set of weights for these basis functions by solving a linear system that encourages the predicted implicit field to have a zero value near the input points, and to have gradients which agree with the input normals. Optionally, we can also predict a geometric mask, which outputs where in space to extract the final surface, trimming away spurious geometry.

\parahead{Predicting a Sparse Voxel Hierarchy} 
Given input points $\Xin = \{\bm x_j^\text{in} \in \RR^3 \}_{j=1}^{n_\text{in}}$, input normals $\Nin = \{\bm n_j^\text{in} \in \RR^3 \}_{j=1}^{n_\text{in}}$, and a voxel size $W$, we first predict a hierarchy of $L$ voxel grids using a convolutional backbone digesting the point cloud with concatenated normal $[\bm x_j^\text{in}, \bm n_j^\text{in}] \in \RR^6$ for each point. 
Each of the predicted voxel grid has $n^{(1)}, \ldots, n^{(L)}$ voxels with widths $W, 2W, \ldots 2^L W$ respectively and any voxel with center $\bm{x}_i^{(l-1)}$ at level $l-1$ is contained within some voxel with center $\bm{x}_j^{(l)}$ at level $l$.
The design of such a backbone network is inspired by \cite{wang2018adaptive} and is described in detail in the Appendix.
We additionally predict features $\bm z_{i}^{(l)} \in \RR^d$ and normals $\bm n_{i}^{(l)} \in \RR^3$ for each voxel in the hierarchy. 
The features $\bm z_i^{(l)}$ are employed to predict a feature field $\phi_\theta^{(l)}(\bm x | \Xin, \Nin)$ which lifts the coordinates $\bm x \in \RR^3$ to $d$-dimenensional vectors via Bézier interpolation followed by an MLP. 
\cref{fig:hierarchy_illustration} shows a 2D illustration of our predicted hierarchy and features. 

\parahead{Sparse Neural Kernel Field Hierarchy}
We encode our reconstructed shape as the zero level set of a 3D implicit field $f_\theta: \RR^3 \rightarrow \RR$ defined as a hierarchical \emph{Neural Kernel Field}, \ie, a weighted combination of \emph{positive definite kernels} which are conditioned on the inputs and centered at the midpoints $\bm x_i^{(l)} \in \RR^3$ of voxels in the predicted hierarchy:
\begin{align}\label{eq:kernel}
    f_\theta(\bm x | \Xin, \Nin) = \sum_{i, l} \alpha_i^{(l)} K^{(l)}_\theta(\bm x, \bm x_i^{(l)} | \Xin, \Nin).
\end{align}
Here, $\alpha_i^{(l)} \in \RR$ are scalar coefficients at the $i^\text{th}$ voxel at level $l = 1, \ldots L$ in the hierarchy, and $K^{(l)}_\theta$ is the predicted kernel for the $l^\text{th}$ level defined as 
\begin{equation}
\begin{aligned}
    K^{(l)}_\theta(\bm x, \bm x') = \langle &\phi_\theta^{(l)}(\bm x; \Xin, \Nin),\\ &\phi_\theta^{(l)}(\bm x'; \Xin, \Nin) \rangle \cdot K^{(l)}_\text{b}(\bm x, \bm x'),
\end{aligned} 
\end{equation}
where $\langle\cdot,\cdot\rangle$ is the dot product, $\phi_\theta^{(l)} : \RR^3 \rightarrow \RR^d$ is the feature field extracted from the hierarchy, and $K^{(l)}_b : \mathbb{R}^3 \times \mathbb{R}^3 \rightarrow \mathbb{R}$ is the \emph{Bézier Kernel}, which decays to zero in a one-voxel (at level-$l$) neighborhood around its origin (See Appendix for definition).

\parahead{Computing a 3D implicit surface from points} Given our predicted voxel hierarchy, learned kernels $K_\theta^{(l)}$, and predicted normals $\bm n_j^{(l)}$, we compute an implicit surface by finding optimal coefficients $\bm \alpha^* = \{\{\alpha^{(l)}_i\}_{l=1}^L\}_{i=1}^{n^{(l)}}$ for the kernel field \eqref{eq:kernel}. 
We find these coefficients by exactly minimizing the following loss in the forward pass of our model (omitting the conditioning of $f_\theta$ on $\Xin, \Nin$ for brevity):
\begin{equation}
\begin{aligned}
    \bm \alpha^* = \argmin_{\alpha_i^{(l)}} & \sum_{l=1}^{L'} \sum_{i = 1}^{n^{(l)}} \|\nabla_{\bm x} f_\theta(\bm x_i^{(l)}) - \bm n_i^{(l)}\|_2^2 + \\&\sum_{j = 1}^{n_\text{in}} |f_\theta(\bm x^\text{in}_j)|^2,
\end{aligned}
\end{equation}
where $L'\leq L$ is a hyper-parameter for the hierarchy.
By minimizing this loss, we want our Neural Kernel Field $f_\theta$ to have a gradient which agrees with the normals at the voxel centers (hence regions around the surface), and to have a scalar value near zero at all the input points $\Xin$. 
Since $f_\theta$ is linear in the parameters $\alpha_i^{(l)}$, we find the optimal coefficients $\bm \alpha^*$ by solving the linear system:
\begin{equation}\label{eq:linsys}
    (\Qmat^\top \Qmat + \Gmat^\top \Gmat) \bm \alpha = \Qmat^\top \bm n,
\end{equation}
where $\bm n$ are the predicted normal vectors $\bm n_i^{(l)}$ stacked into a single vector, $\bm \alpha$ is the vector of coefficients $\alpha_i^{(l)}$, and 
\begin{equation}
    \begin{aligned}
        \Gmat = \begin{bmatrix} \Gmat^{(1)} & \ldots & \Gmat^{(L)} \end{bmatrix}, \\
        \Qmat = \begin{bmatrix} \Qmat^{(1)} & \ldots & \Qmat^{(L)} \end{bmatrix},
    \end{aligned}
\end{equation}
\begin{equation}
    \Gmat^{(l)}_{i, j} = K_\theta(\bm x^\text{in}_i, \bm x_j^{(l)}), \quad \Qmat^{(l)}_{i, j} = \partial_{\bm x_i^{(l')}} K_\theta(\bm x_i^{(l')}, \bm x_j^{(l)})
\end{equation}
are the gram matrix and partial derivatives of the gram matrix at the voxel centers where normals are defined, respectively.

We remark that the linear system~\eqref{eq:linsys} is sparse due to modulation with the compactly supported $K_b^{(l)}$, and positive definite by construction since it is a Gram matrix. As a result, \eqref{eq:linsys} can be solved very efficiently on a GPU.

\parahead{Masking module} The predicted Neural Kernel Field $f_\theta$ is defined on the entire voxel hierarchy, however at coarse levels far from the surface, it may contain unwanted geometry. To discard such geometry away from the predicted surface, we add an additional branch to our backbone as $\varphi : \RR^3 \rightarrow \{0, 1\}$ which determines if a point $\bm x$ should be trimmed ($\varphi(\bm x) = 0$) or kept ($\varphi(\bm x) = 1$). 
The branch originates from the immediate features of the backbone network and consists of a few linear layers with ReLU activations followed by a sigmoid. When we extract the final surface, we only consider vertices in regions where $\varphi(\bm x) > 0.5$.

\subsection{Supervision}
To train our model, we require pairs $(\Xin = \{\bm x_i \in \RR^3\}, \Oin = \{\bm o_i \in \RR^3\})$ and $(\Xdense = \{\bm x_j\}, \bm O_\text{dense} = \{\bm o_j\})$.
Here $\Xin$ and $\Xdense$ are noisy input points and dense supervision points respectively, and $\Oin$ and $\bm O_\text{dense}$ are sensor origin for each input and supervision point (\textit{i.e.} a position in 3D space from which each point was acquired). 
We additionally compute input and supervision normals $\Nin, \bm N_\text{dense}$ by fitting planes to points in a local neighborhood and orienting the normals to align with the directions from points to sensors.
We remark that our training requirements impose no restrictions on the shapes being trained on. 
For example, one could use a single LiDAR frame as input and an accumulated LiDAR scan of a scene as supervision, alongside a noisy scan of a synthetic object as input and a dense noiseless scan of the same object as supervision.
\setlength{\columnsep}{5pt}
\setlength{\intextsep}{0pt}
\begin{wrapfigure}[7]{r}{0.35\linewidth}
\begin{center}
\includegraphics[width=\linewidth]{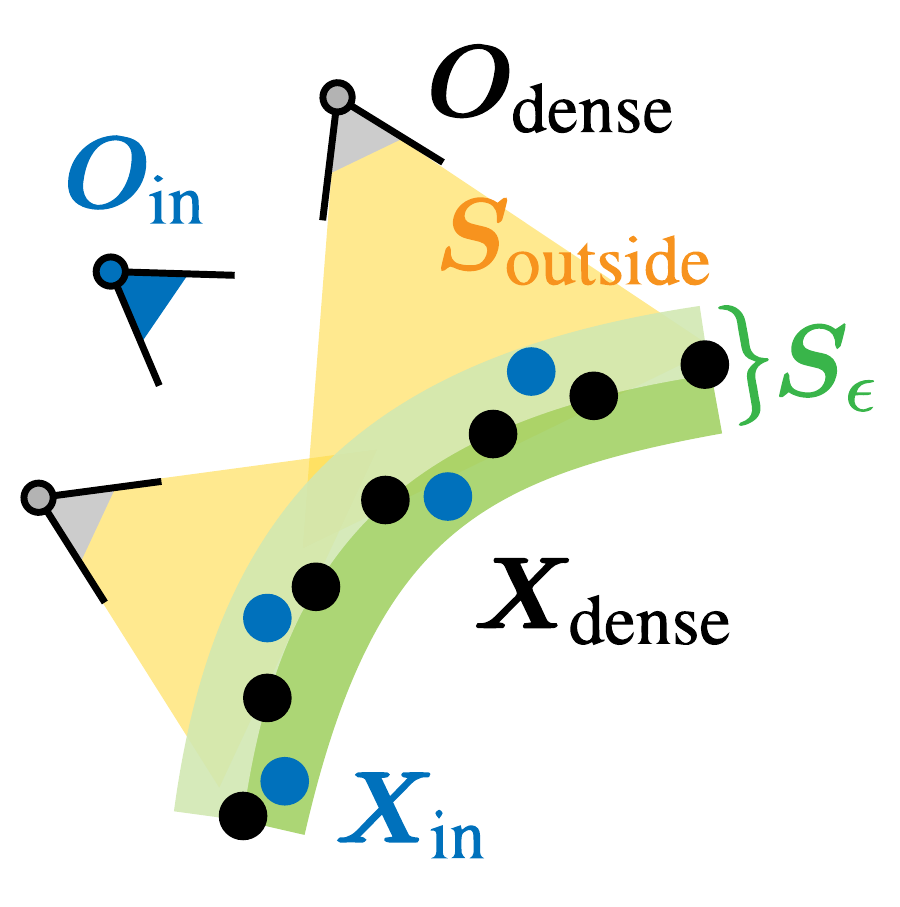}
\end{center}
\end{wrapfigure}


In order to define the loss terms used to supervise our model, we first define two regions of space around the dense points $\Xdense$:
\begin{itemize}[topsep=0pt]
    \item $\Seps$: points which are $\epsilon$ distance or less from $\Xdense$ \textit{i.e.} $\{\bm x | \min_{\bm x_j \in \Xdense} \|\bm x - \bm x_j\|_2 < \epsilon \}$,
\end{itemize}

\begin{itemize}[topsep=0pt]
    \item $\bm S_\text{outside}$: points which lie within the region enclosing points in $\Xdense$ and their sensor origin in $\bm O_\text{dense}$.
\end{itemize}


Then, given a predicted Neural Kernel Field $f_\theta(\bm{x})$, we backpropagate through the following loss functions:
\begin{itemize}
    \item $\mathcal{L}_\text{surf}(f) = \mathbb{E}_{\bm x \in \Xdense}[\|f(\bm x)\|_1]$;
    \item $\mathcal{L}_\text{tsdf}(f) = \mathbb{E}_{\bm x \in \Seps}\big[\|f(\bm x) - \text{tsdf}(\bm{x}, \Xdense)\|_1\big]$ where $\text{tsdf}(\bm x, \Xdense)$ is the ground-truth truncated signed distance computed from $\Xdense$ using nearest neighbors;
    \item $\mathcal{L}_\text{normal}(f) = \mathbb{E}_{\bm n \in \Ndense} \bigg[1 - \bigg\langle \frac{\nabla_{\bm{x}} f(\bm x)}{\|\nabla_{\bm{x}} f(\bm x)\|_2}, \bm n \bigg\rangle \bigg]$;
    \item $\mathcal{L}_\text{outside}(f) = \mathbb{E}_{\bm x \in \bm S_\text{outside}} e^{-\beta \|f(\bm x)\|_1}$, where $\beta = 0.1$;
    \item $\mathcal{L}_\text{min-surf}(f) = \mathbb{E}_{\bm x \in \Seps} \bigg[\frac{\eta \pi^{-1}}{\eta^2 + f(\bm x)^2} \bigg]$, where $\eta = 0.5$.
\end{itemize}

Here $\mathcal{L}_\text{surf}$ ensures that the implicit function is zero near the ground truth surface, $\mathcal{L}_\text{tsdf}$ ensures that the implicit field undergoes a sign change near the surface,
$\mathcal{L}_\text{normal}$ ensures the gradient of the predicted implicit agrees with the dense normals, $\mathcal{L}_\text{outside}$ ensures there is no geometry far away from the surface, and $\mathcal{L}_\text{min-surf}$ acts as a regularizer encouraging the predicted implicit surface to have minimal area~\cite{zhang2022critical}. 
The latter two losses are omitted if no sensor-based information is available.

We additionally compute structure prediction and masking losses which we describe in the Appendix.
We train our model in an end-to-end fashion using gradient descent by back-propagating through the sum of all these loss functions.




\input{figures/hierarchy_illustration}

%% file: figures/hierarchy_illustration.tex
\begin{figure}
    \centering
    \includegraphics[width=\columnwidth]{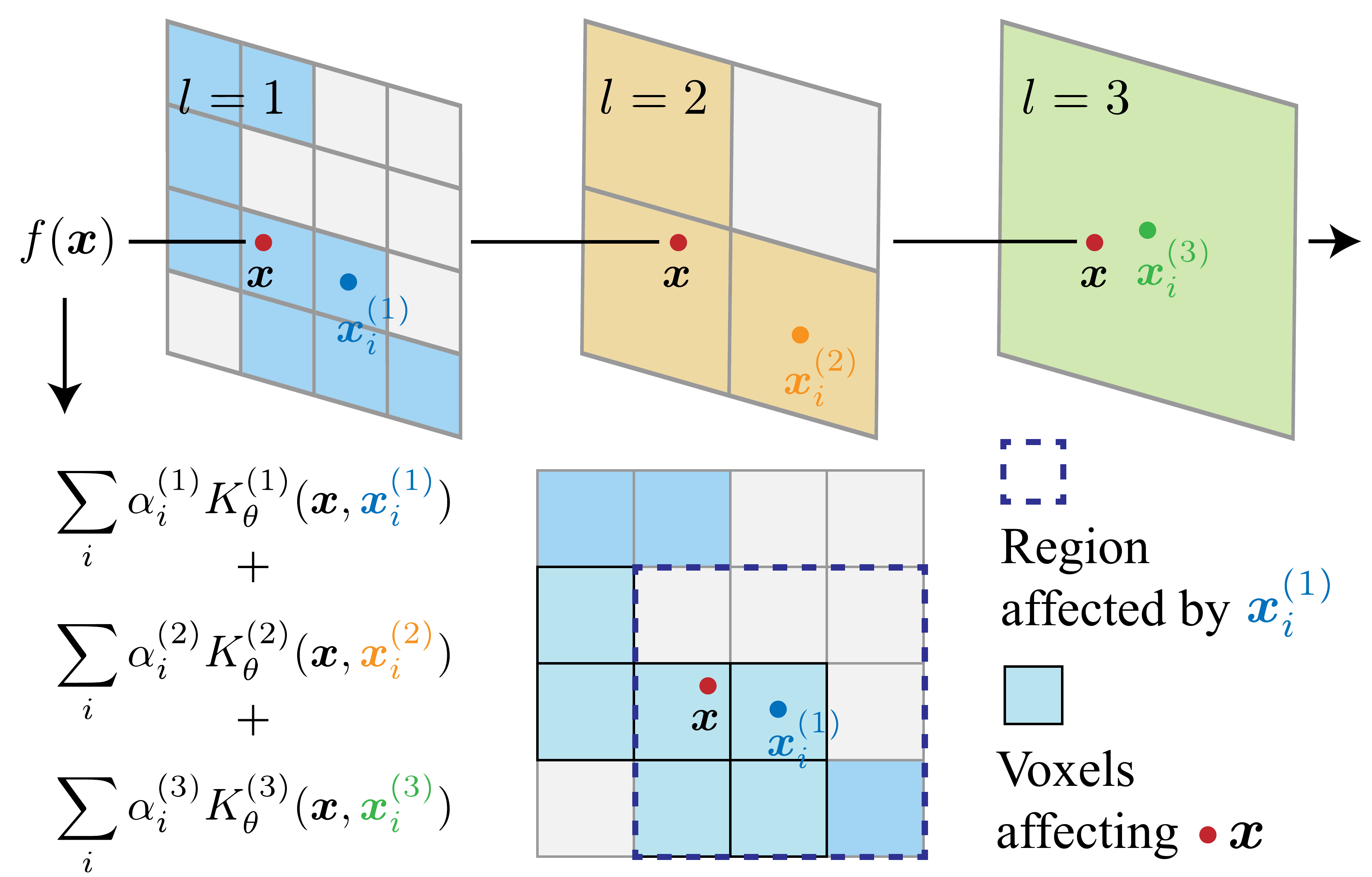}
    \vspace{-5mm}
    \caption{\textbf{Our implicit field $f(\bm x)$} is represented as a sum of kernel basis functions on a sparse voxel hierarchy. Each voxel with center $\bm x_i^{(l)}$ contributes one kernel basis function $K_\theta^{(i)}(\bm x, \bm x_i^{(l)})$ with support in the one-ring around $\bm x_i^{(l)}$.}
    \label{fig:hierarchy_illustration}
    \vspace{-2mm}
\end{figure}

%% file: sections/sec4-exp.tex
\section{Experiments}

\parahead{Overview}
In this section we demonstrate that \MethodName{} fulfills the three main desired properties of practical surface reconstruction method as analyzed in \cref{sec:introduction}:
(1) \emph{Accuracy} (\cref{subsec:exp:object}), by training and testing on object-level datasets~\cite{chang2015shapenet,Koch_2019_CVPR,Thingi10K} with varying noise settings.
(2) \emph{Scalability} (\cref{subsec:exp:scale}), by evaluating on large-scale outdoor driving dataset~\cite{Dosovitskiy17}.
(3) \emph{Generalizability} (\cref{subsec:exp:general}), where we train on object-level/outdoor datasets and test on room-level datasets~\cite{dai2017scannet,Matterport3D} as well as scans with very low densities. 
Notably, to encourage the practical usage of \MethodName{}, we present a \emph{kitchen-sink-model} (denoted as `Ours - \ks{}'%
) trained on the union of various datasets \cite{chang2015shapenet,Koch_2019_CVPR,Dosovitskiy17,Matterport3D} and report its performance whenever applicable. 
While this model slightly underperforms domain-specific models, it still outperforms most baseline methods and can be used on a wide variety of inputs as shown in \cref{fig:teaser} and \cref{fig:waymo}.
We hope the kitchen-sink-model enables end-users to use NKSR in a plug-and-play manner.

\parahead{Implementation Details}
Our pipeline is fully accelerated using PyTorch and CUDA. 
The operations on our sparse hierarchy including convolution, neighborhood querying and interpolations are based on a customized tree structure that is highly efficient and scalable. 
Our sparse linear solver uses a Jacobi-preconditioned conjugate gradient method and works jointly with the sparse hierarchy for fast inference.
Unless otherwise specified, our experiments are run on a single V100 GPU with 8 CPU cores.
Hyperparameter details are given in the Appendix.

\subsection{Accuracy: Object-level Reconstruction}
\label{subsec:exp:object}

\begin{figure}[!t]
    \centering
    \includegraphics[width=\linewidth]{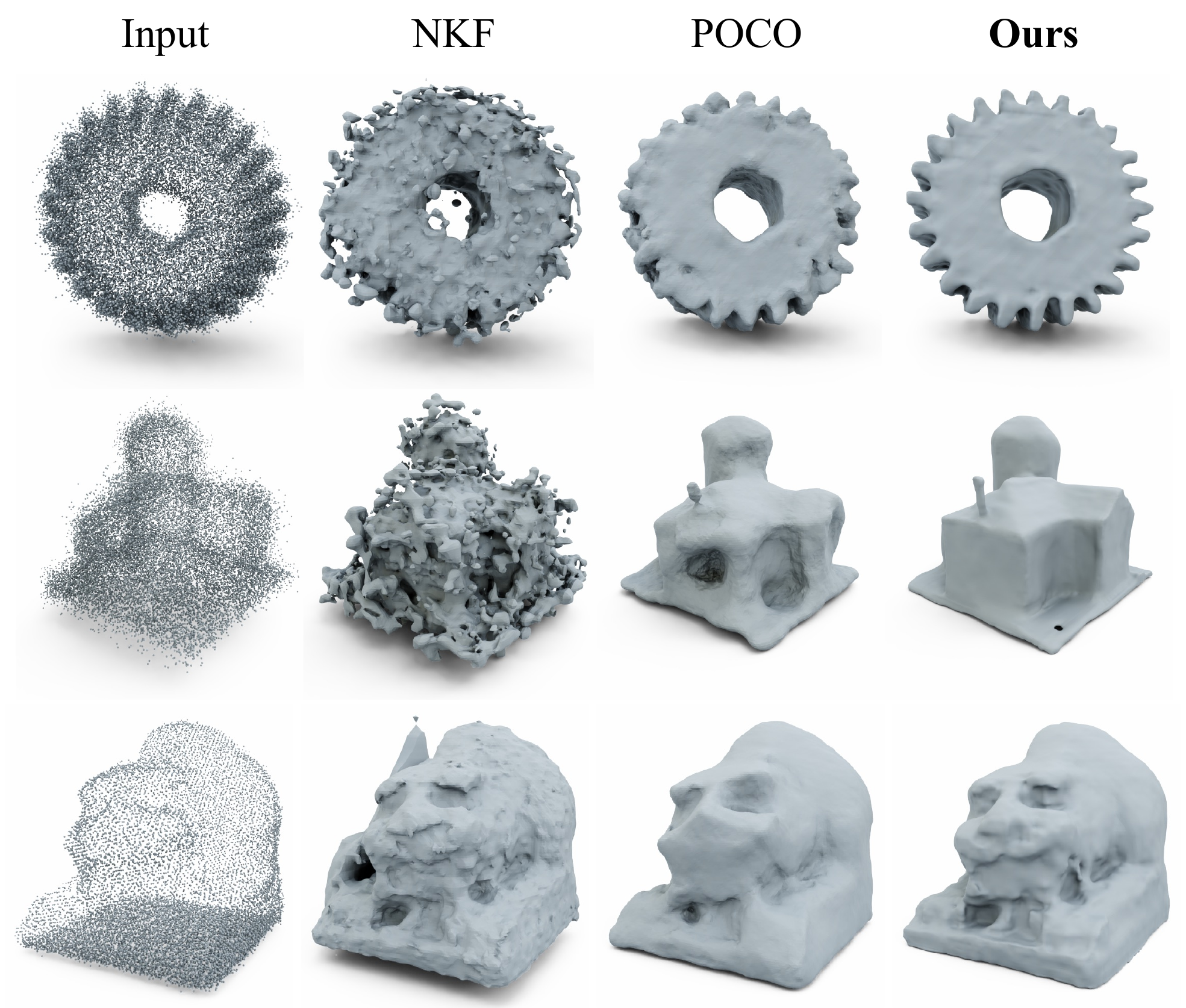}
    \vspace{-1.5em}
    \caption{\textbf{ABC/Thingi10K~\cite{Koch_2019_CVPR,Thingi10K} visualization.}}
    \label{fig:abc}
    \vspace{-0.5em}
\end{figure}

\begin{figure}[!t]
    \centering
    \includegraphics[width=\linewidth]{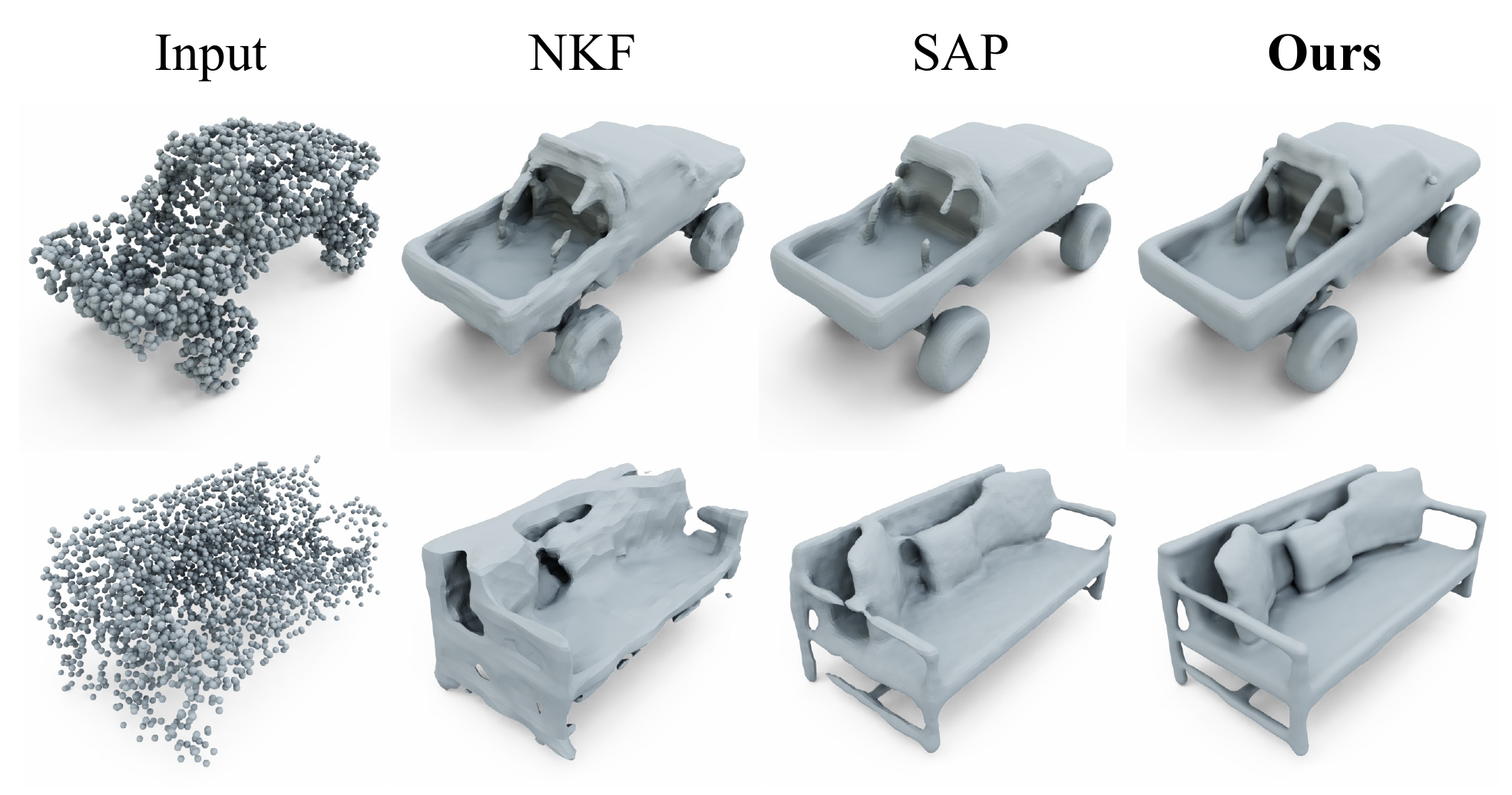}
    \vspace{-1.5em}
    \caption{\textbf{ShapeNet~\cite{chang2015shapenet} visualization.} The two shapes are with $\sigma=0.005$ and $\sigma=0.025$ Gaussian noise respectively.}
    \label{fig:shapenet}
    \vspace{-1.5em}
\end{figure}

\input{tables/abc}

\input{tables/shapenet}

\parahead{Settings}
We follow two common evaluation settings from the literature.
One is the manifold ShapeNet~\cite{chang2015shapenet} dataset prepared by \cite{mescheder2019occnet}. The dataset contains man-made geometries from 13 categories, with $>$30K shapes for training and $>$8K shapes for testing.
Gaussian noise of different standard deviations (denoted as $\sigma$) is added to the randomly-subsampled points from the ground truth as input.
As many existing learning-based baselines do not need point normals $\Nin$ as input, we present a variant of our model that does not take $\Nin$ as extra input channels for a fair comparison.
The other setting is from \cite{erler2020points2surf} where a random subset of $\sim$5K shapes from ABC~\cite{Koch_2019_CVPR} is picked for training and testing, and an additional 100 shapes from Thingi10K~\cite{Thingi10K} is used for testing generalization. The input is acquired by simulating ToF sensors with different levels of noise and densities.
For the metrics we use the standard Chamfer distance ($d_C$), F-score (F-S.), normal consistency (N.C.), and intersection-over-union ratio (IoU) as benchmarks.

\parahead{Results}
The comparisons are quantitatively shown in \cref{tab:abc} and \cref{tab:shapenet}, and selectively visualized in \cref{fig:abc} and \cref{fig:shapenet}.
Our model reaches state-of-the-art performance on all the datasets.
Our baseline, NKF, works well on the noise-free setting but inelegantly degrades with higher noise due to its over-reliance on the raw input normals.
On the other hand, SAP and NGSolver are more robust under noise, but the fitting tightness as reflected by $d_C$ is higher than ours due to the lack of representation power.
Our performance gain is mainly based on the gradient-based energy fitting formulation backed up by the natural inductive biases emerging from the learned kernel.

\subsection{Scalability: Outdoor Driving Scenes}
\label{subsec:exp:scale}

\input{tables/carla}

\begin{figure*}[!t]
    \centering
    \includegraphics[width=\linewidth]{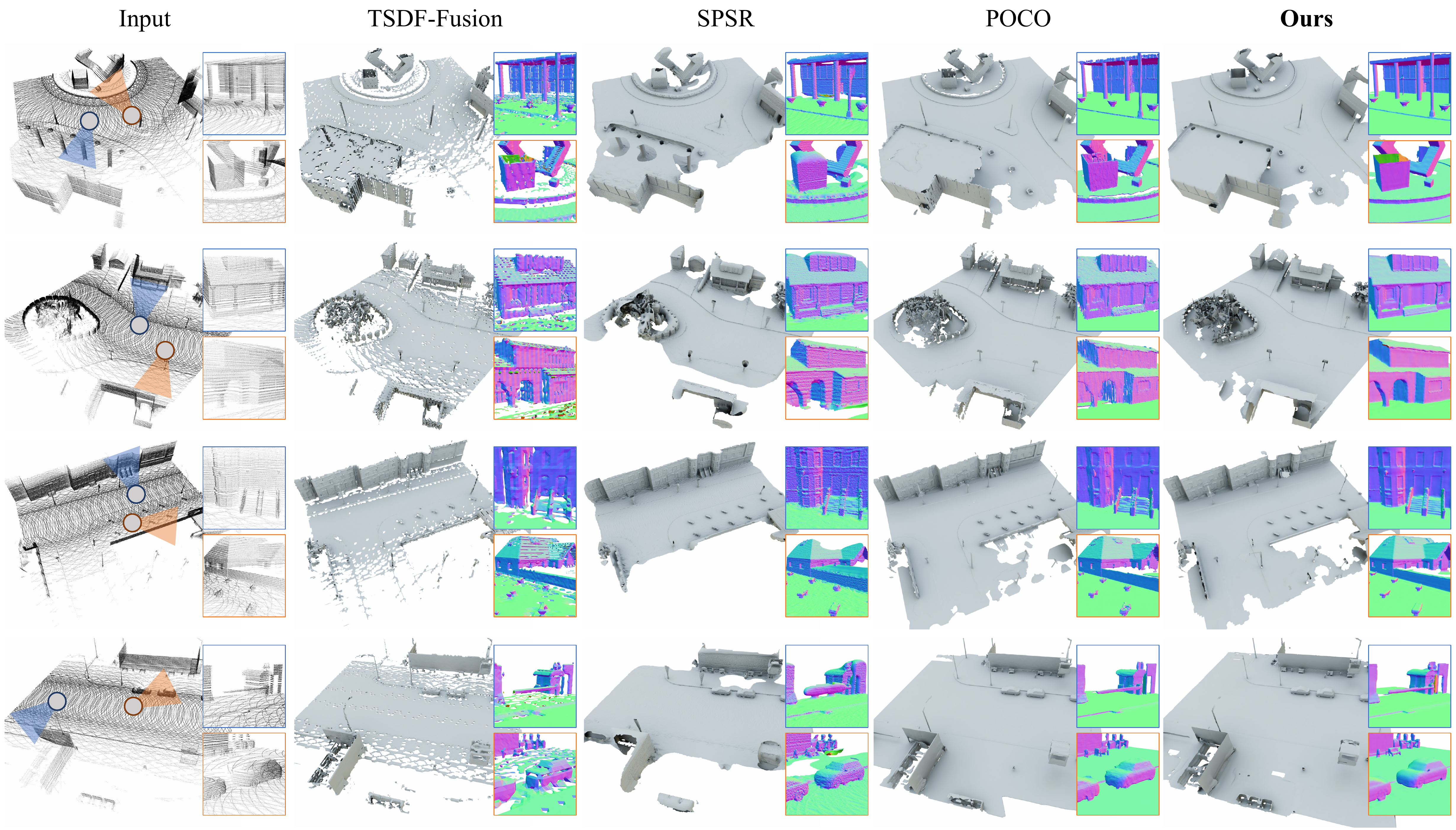}
\vspace{-2em}
    \caption{\textbf{CARLA~\cite{Dosovitskiy17} visualization.} The insets show zoom-ins captured by the cameras shown in the leftmost column with the corresponding color. The upper 2 rows are from the `Novel' subset and the lower 2 rows are from the `Original' subset.}
    \label{fig:carla}
    \vspace{-1em}
\end{figure*}

\parahead{Dataset}
The applicability of \MethodName{} to large-scale datasets is demonstrated using the synthetic CARLA~\cite{Dosovitskiy17} dataset due to the lack of large-scale real-world datasets with ground-truth geometries.
To generate such a dataset, we manually pick 3 towns and simulate 10 random drives using the engine.
We call these drives the `Original' subset.
An additional town along with its 3 drives is used only during evaluation to test generalization, which we denote as the `Novel' subset.
For $(\Xin, \Oin)$, we use a sparse 32-beam LiDAR with 0-5cm ray distance noise and 0-3$^\circ$ pose noise. For $(\Xdense, \Odense)$, we employ a noise-free highly-dense 256-beam LiDAR for ground-truth supervision.
The accumulated LiDAR points are cropped into $51.2 \times 51.2 \text{m}^2$ chunks for the ease of benchmarking.
Please find more details and visualizations in the Appendix.

\input{tables/room}

\begin{figure}[!t]
    \centering
    \includegraphics[width=\linewidth]{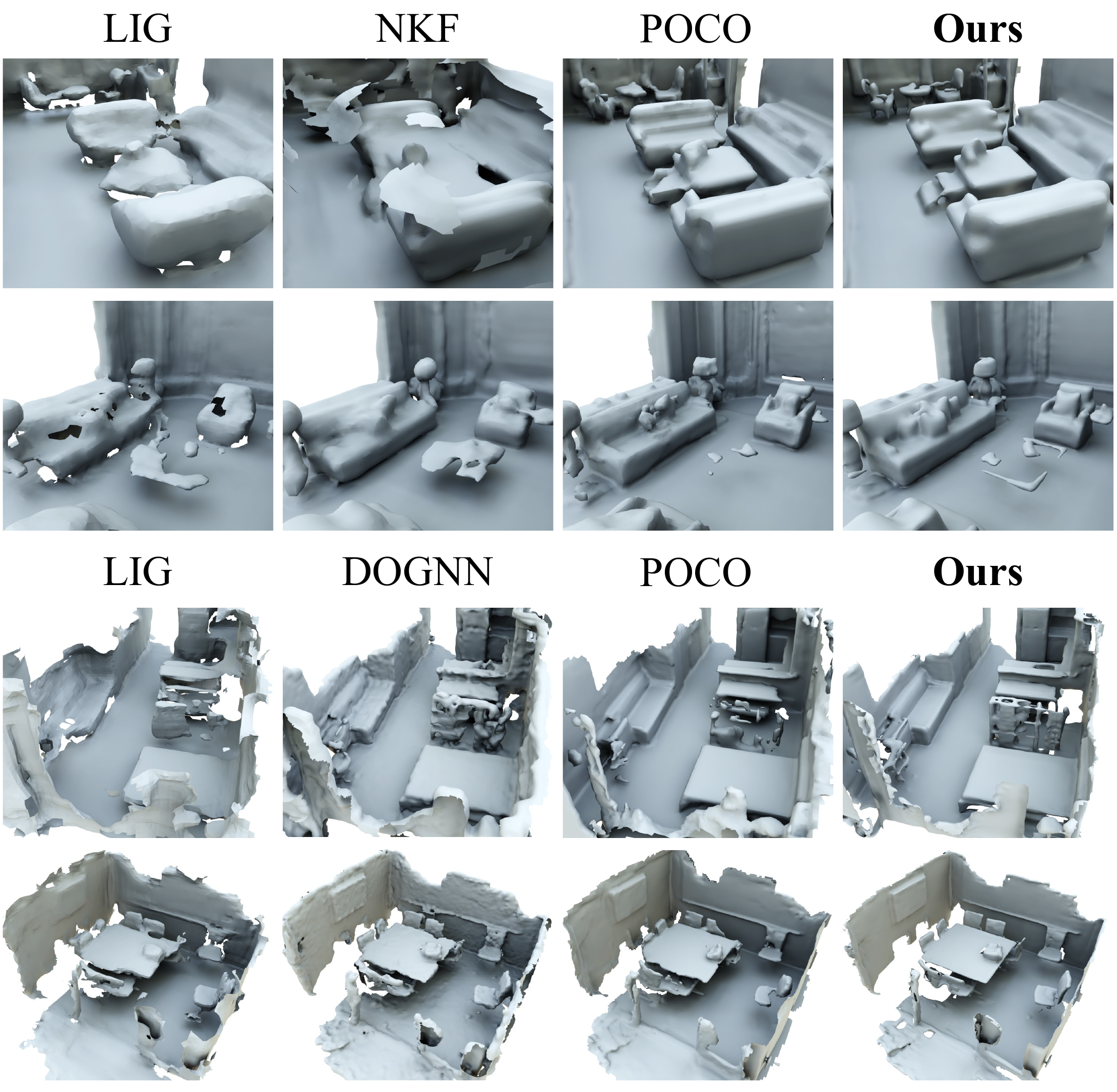}
    \vspace{-2em}
    \caption{\textbf{Room-level datasets~\cite{Matterport3D,dai2017scannet} visualization.} All the models are trained only with ShapeNet.}
    \label{fig:room}
    \vspace{-2em}
\end{figure}

\begin{figure*}[!t]
    \centering
    \includegraphics[width=\linewidth]{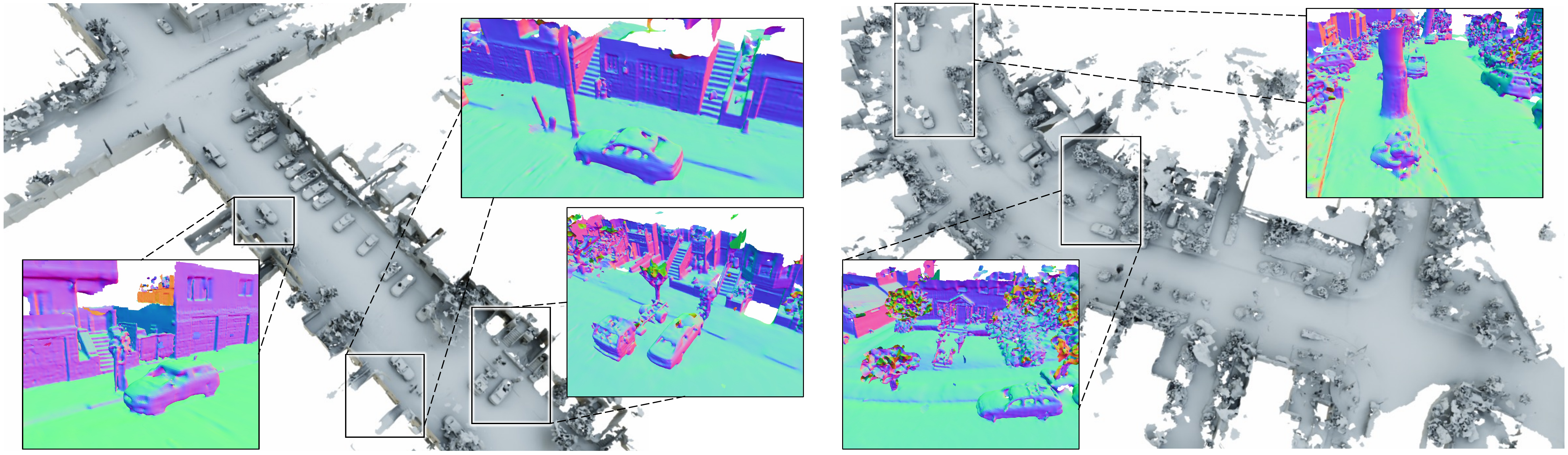}
    \caption{\textbf{Application to Waymo~\cite{Sun_2020_CVPR} dataset.} We run our kitchen-sink-model in an out-of-core manner (see Appendix for implementation details) to scale to very large scenes consisting of 10M / 11M (left / right) points, taking only 20s / 35s.}
    \label{fig:waymo}
    \vspace{-1em}
\end{figure*}

\parahead{Results}
We compare our results to TSDF-Fusion, SPSR and the learning-based POCO.
While for the latter two baselines we use the same voxel sizes $W=10$cm as ours, for TSDF-Fusion we find it necessary to increase $W$ to $30$cm to reach decent surface completeness.
The results are shown in \cref{tab:carla} and visualized in \cref{fig:carla}.
We compute single-sided Chamfer distance that reflects reconstruction accuracy (Acc.) and completeness (Comp.). We additionally report average running times for each method on the datasets. The mean/min/max number of input points in this setting are 490k/290k/820k.
Compared to ours, SPSR is quite sensitive to the noise and sparsity in the input, leaving bumpy and incomplete geometries.
Although POCO could reach a similar completeness value, the fitted surfaces fail to faithfully respect the input. The long running time (161x slower than ours) also prohibits POCO from practical use.

\vspace{0.2em}
\subsection{Generalization across Domains and Densities}
\label{subsec:exp:general}
\vspace{-0.8em}

\parahead{Across domains}
We compare the generalizability of our method with others by directly applying the models trained on ShapeNet and Synthetic Room dataset (Synth. Rooms)~\cite{peng2020convoccnet} to room-level datasets, \ie, ScanNet~\cite{dai2017scannet} and the test split of Matterport3D~\cite{Matterport3D}.
For both datasets we sample 10K points as input, and normalize the scale to roughly match the training set.
As shown by the comparisons in \cref{fig:room} and \cref{tab:room}, the generalization of our method is significantly better than the baselines, with ShapeNet training set reaching the highest accuracy possibly due to its diversity and similar geometric distributions.

\parahead{Across densities}
We test the robustness of our model under sparse input by keeping only one scan of LiDAR frame within a fixed driving distance in our CARLA dataset (`Novel' subset).
The results are shown in \cref{fig:density} and \cref{tab:density}.
At the level of extreme sparsity our method is still able to reconstruct complete geometry (\eg the ground) while the baselines start to degrade.

\begin{figure}[!t]
    \centering
    \includegraphics[width=\linewidth]{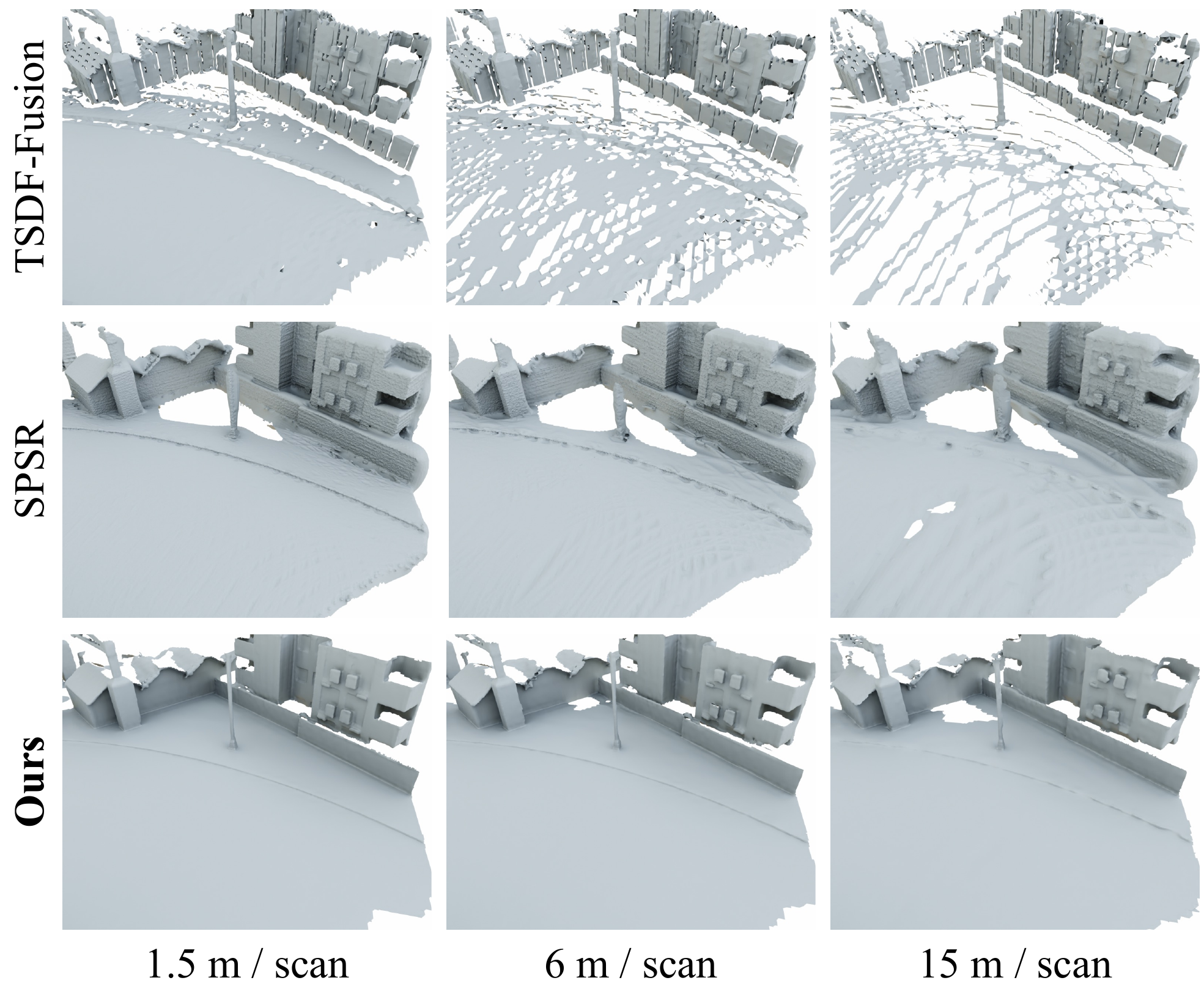}
    \vspace{-2em}
    \caption{\textbf{Generalization to different densities.}}
    \label{fig:density}
    \vspace{-1.5em}
\end{figure}

\input{tables/density}

\subsection{Ablation Study}

We run our method with different feature dimensions $d$ for kernel computation, as well as different voxel sizes $W$, and the results are shown in \cref{fig:ablation}.
While increasing the feature dimension helps reach a slightly better performance, the influence of voxel sizes is more prominent.
We try to remove the hierarchies from the linear solver by setting $\{\alpha_i^{(l)} \,|\, l>0\}$ to 0 (`w/o Hier.'), or remove the gradient-based matrices $\Qmat^\top \Qmat$ (`w/o Grad.').
Both of the settings lead to a degraded performance, showing the effectiveness of our design choices. 

\input{figures/ablation.tex}

%% file: tables/abc.tex
\begin{table*}[!thbp]
\setlength{\tabcolsep}{4.65pt}
\centering
\small
\caption{\textbf{ABC/Thingi10K~\cite{Koch_2019_CVPR,Thingi10K} comparison}. $d_C$ is multiplied by $10^3$. $\sigma$ is the Gaussian noise added to the sensor depth and $L$ is the largest box length. 10 scans are used to accumulate the point cloud unless specified.}
\label{tab:abc}
\vspace{-1em}
\begin{tabular}{l|cccccc|cccccccccc}
\toprule
          & \multicolumn{6}{c|}{ABC (100 shapes)}                                                                                                                & \multicolumn{10}{c}{Thingi10K (100 shapes)}                                                                                                                                                                                                                    \\ \midrule
          & \multicolumn{2}{c|}{$\sigma=0$}                          & \multicolumn{2}{c|}{$\sigma\in [0,0.05L]$}                         & \multicolumn{2}{c|}{$\sigma=0.05L$}    & \multicolumn{2}{c|}{$\sigma=0$}                          & \multicolumn{2}{c|}{$\sigma=0.01L$}                         & \multicolumn{2}{c|}{$\sigma=0.05L$}                         & \multicolumn{2}{c|}{5 scans}                        & \multicolumn{2}{c}{30 scans}     \\
          & $d_C$$\downarrow$         & \multicolumn{1}{c|}{F-S.$\uparrow$}            & $d_C$$\downarrow$         & \multicolumn{1}{c|}{F-S.$\uparrow$}            & $d_C$$\downarrow$         & F-S.$\uparrow$            & $d_C$$\downarrow$         & \multicolumn{1}{c|}{F-S.$\uparrow$}            & $d_C$$\downarrow$         & \multicolumn{1}{c|}{F-S.$\uparrow$}            & $d_C$$\downarrow$         & \multicolumn{1}{c|}{F-S.$\uparrow$}            & $d_C$$\downarrow$         & \multicolumn{1}{c|}{F-S.$\uparrow$}            & $d_C$$\downarrow$         & F-S.$\uparrow$            \\ \midrule
SPSR~\cite{kazhdan2013screened}      & 7.02          & \multicolumn{1}{c|}{87.5}          & 11.2          & \multicolumn{1}{c|}{72.8}          & 18.8          & 47.9          & 4.23          & \multicolumn{1}{c|}{91.9}          & 5.44          & \multicolumn{1}{c|}{90.3}          & 16.5          & \multicolumn{1}{c|}{52.5}          & 12.4          & \multicolumn{1}{c|}{77.6}          & 3.07          & 96.7          \\
POCO~\cite{boulch2022poco}      & 5.34          & \multicolumn{1}{c|}{88.3}          & 8.23          & \multicolumn{1}{c|}{75.7}          & 12.0          & 58.9          & 4.42          & \multicolumn{1}{c|}{92.5}          & 5.10          & \multicolumn{1}{c|}{89.7}          & 11.2          & \multicolumn{1}{c|}{58.8}          & 6.95          & \multicolumn{1}{c|}{84.4}          & 3.69          & 95.0          \\
SAP~\cite{peng2021shape}       & 6.83          & \multicolumn{1}{c|}{85.0}          & 8.00          & \multicolumn{1}{c|}{79.5}          & 10.4          & 68.7          & 4.30          & \multicolumn{1}{c|}{92.7}          & 4.54          & \multicolumn{1}{c|}{91.8}          & 7.82          & \multicolumn{1}{c|}{\textbf{74.7}} & 6.73          & \multicolumn{1}{c|}{84.7}          & 3.95          & 93.8          \\
NKF~\cite{williams2022neural}       &  6.10         & \multicolumn{1}{c|}{88.1}              &  13.8        & \multicolumn{1}{c|}{62.3}              &  24.0        & 35.1         &  3.48             & \multicolumn{1}{c|}{94.2}              &  4.78        & \multicolumn{1}{c|}{90.8}              &  24.7      & \multicolumn{1}{c|}{34.3}              & 7.05       & \multicolumn{1}{c|}{84.8}              &  4.36      &  93.2        \\
{\scriptsize NGSolver~\cite{huang2022neuralgalerkin}}       &  3.92    & \multicolumn{1}{c|}{92.7}   &  6.35        & \multicolumn{1}{c|}{83.1}              &  9.68     & 66.4         &  2.96             & \multicolumn{1}{c|}{95.9}              &  3.51        & \multicolumn{1}{c|}{95.0}              &  8.70      & \multicolumn{1}{c|}{69.1}              & 5.65       & \multicolumn{1}{c|}{89.2}              &  2.80      &  97.1        \\
\textbf{Ours}      & \textbf{3.68} & \multicolumn{1}{c|}{\textbf{93.2}} & \textbf{6.00} & \multicolumn{1}{c|}{\textbf{85.4}} & \textbf{8.70} & \textbf{73.2} & \textbf{2.36} & \multicolumn{1}{c|}{\textbf{97.3}} & \textbf{3.19} & \multicolumn{1}{c|}{\textbf{95.9}} & \textbf{7.66} & \multicolumn{1}{c|}{\textbf{74.7}} & \textbf{5.10} & \multicolumn{1}{c|}{\textbf{89.9}} & \textbf{2.48} & \textbf{98.0} \\ \midrule
\textbf{Ours} - \ks{} & 4.10          & \multicolumn{1}{c|}{92.2}          & 6.44          & \multicolumn{1}{c|}{83.6}          & 9.97          & 68.1          & 2.92          & \multicolumn{1}{c|}{96.3}          & 3.34          & \multicolumn{1}{c|}{95.6}          & 8.55          & \multicolumn{1}{c|}{72.7}          & 5.60          & \multicolumn{1}{c|}{89.1}          & 2.54          & 97.7          \\ \bottomrule
\end{tabular}
\vspace{-1em}
\end{table*}

%% file: tables/shapenet.tex
\begin{table}[!t]
\setlength{\tabcolsep}{4.2pt}
\centering
\small
\caption{\textbf{ShapeNet~\cite{chang2015shapenet} comparison}. `N.' denotes whether normals $\Nin$ are used as input. $d_C$ is multiplied by $10^3$.}
\label{tab:shapenet}
\vspace{-1em}
\begin{tabular}{l|c|cc|cc|cc}
\toprule
                                   &                   & \multicolumn{2}{c|}{\begin{tabular}[c]{@{}c@{}}1000 Pts.\\ $\sigma=0.0$\end{tabular}} & \multicolumn{2}{c|}{\begin{tabular}[c]{@{}c@{}}3000 Pts.\\ $\sigma=0.005$\end{tabular}} & \multicolumn{2}{c}{\begin{tabular}[c]{@{}c@{}}3000 Pts.\\ $\sigma=0.025$\end{tabular}} \\ \cmidrule(l){2-8} 
                                   & N.                & $d_C$$\downarrow$                          & IoU$\uparrow$                            & $d_C$$\downarrow$                     & IoU$\uparrow$                                    & $d_C$$\downarrow$                             & IoU$\uparrow$                                       \\ \midrule
ConvONet~\cite{peng2020convoccnet} & \multirow{5}{*}{-} & 6.07                                      & 82.3                                      & 4.35                                       & 88.0                                       & 7.31                                       & 78.7                                      \\
IMLSNet~\cite{liu2021deep}         &                   & 3.15                                      & 91.2                                      & 3.03                                       & 91.3                                       & 6.58                                       & 76.0                                      \\
SAP~\cite{peng2021shape}           &                   & 3.44                                      & 90.8                                      & 3.30                                       & 91.1                                       & 5.34                                       & \textbf{82.9}                             \\
POCO~\cite{boulch2022poco}         &                   & 3.03                                      & 92.7                                      & 2.93                                       & 92.2                                       & 5.82                                       & 81.7                                      \\
NGSolver~\cite{huang2022neuralgalerkin}      &        & 2.91                                      & 91.9                                      & 2.90                                       & 91.8                                       & 5.06                                       & 82.8                                      \\
\textbf{Ours}                      &                   & \textbf{2.64}                             & \textbf{93.4}                             & \textbf{2.71}                              & \textbf{92.6}                              & \textbf{4.96}                              & \textbf{82.9}                                      \\ \midrule
SPSR~\cite{kazhdan2013screened}    & \multirow{4}{*}{\checkmark} & 6.26                                      & 81.4                                      & 3.84                                       & 88.5                                       & 10.7                                       & 66.8                                      \\
SAP~\cite{peng2021shape}           &                   & 3.21                                      & 92.1                                      & 3.16                                       & 92.3                                       & 4.44                                       & 87.1                                      \\
NKF~\cite{williams2022neural}      &                   & 2.65                                      & 94.7                                      & 3.17                                       & 91.2                                       & 11.7                                       & 67.0                                      \\
NGSolver~\cite{huang2022neuralgalerkin}      &        & 2.47                                      & 95.0                                      & 2.51                                       & 94.1                                       & 3.93                                       & 87.5                                      \\
\textbf{Ours}                      &                   & \textbf{2.34}                             & \textbf{95.6}                             & \textbf{2.45}                              & \textbf{94.3}                              & \textbf{3.87}                              & \textbf{87.6}                             \\ \bottomrule
\end{tabular}
\vspace{-2em}
\end{table}

%% file: tables/carla.tex
\begin{table}[!t]
\setlength{\tabcolsep}{3.pt}
\centering
\small
\caption{\textbf{CARLA~\cite{Dosovitskiy17} comparison.} $d_C$ is the average of Acc. and Comp. (Unit is cm. The smaller the better.)}
\label{tab:carla}
\vspace{-1em}
\begin{tabular}{@{}l|ccc|ccc|c@{}}
\toprule
            & \multicolumn{3}{c|}{Original}               & \multicolumn{3}{c|}{Novel}                  & Time                 \\
            & Acc.         & Comp.        & F-S.$\uparrow$            & Acc.         & Comp.        & F-S.$\uparrow$            & (sec.)               \\ \midrule
{\scriptsize TSDF-Fusion~\cite{vizzo2022sensors}} & 8.2          & 8.0          & 80.2          & 8.6          & 6.6          & 80.7          & \textbf{0.5}         \\
POCO~\cite{boulch2022poco}        & 10.5          & 3.6         & 90.1          & 9.1          & 2.9         & 92.4          & 420                  \\
SPSR~\cite{kazhdan2013screened}        & 10.3         & 16.4         & 86.5          & 9.9          & 12.8         & 88.3          & 30                   \\
\textbf{Ours}        & \textbf{5.6}          & \textbf{2.2} & \textbf{93.9}          & \textbf{3.6} & \textbf{2.1} & \textbf{96.0} & 2.6 \\ \midrule
\textbf{Ours} - \ks{}   & 4.1 & 3.0          & 94.0 & 3.6 & 2.4          & 96.0 &   2.6              \\ \bottomrule
\end{tabular}
\vspace{-2em}
\end{table}

%% file: tables/room.tex
\begin{table}[!t]
\setlength{\tabcolsep}{3.3pt}
\centering
\small
\caption{\textbf{Room-level dataset~\cite{dai2017scannet,Matterport3D} comparison}. $d_C$ is multiplied by $10^3$.}
\label{tab:room}
\vspace{-1em}
\begin{tabular}{l|c|ccc|ccc}
\toprule
              & \multirow{2}{*}{\begin{tabular}[c]{@{}c@{}}Training\\ Set\end{tabular}} & \multicolumn{3}{c|}{ScanNet}                  & \multicolumn{3}{c}{Matterport3D}              \\
              &                                                                 & $d_C$$\downarrow$         & F-S.$\uparrow$            & N.C.$\uparrow$  & $d_C$$\downarrow$         & F-S.$\uparrow$            & N.C.$\uparrow$          \\ \midrule
SPSR~\cite{kazhdan2013screened}          & -                                                                    & 7.04          & 84.3          & 87.2          & 10.4          & 87.0          & 92.3          \\ \midrule
LIG~\cite{huang2021di}     & \multirow{5}{*}{\begin{tabular}[c]{@{}c@{}}Shape\\ Net\end{tabular}}                                                          & 6.19          & 83.8          & 83.7          & 5.13          & 90.1          & 90.1          \\
POCO~\cite{boulch2022poco}          &                                                                           & 6.21          & 77.4          & 87.1          & 5.14          & 84.9          & 93.7          \\
NKF~\cite{williams2022neural}           &                                                                       & 6.50        & 80.9          & 84.2          & 6.48         &  84.2         &   90.4  \\
{\scriptsize DOGNN~\cite{wang2022dual}}           &                                                                       & 4.93        & 85.9          & 85.7          & 4.85         &  89.3         &   92.4  \\
\textbf{Ours} &                                                                                                 & \textbf{2.68} & \textbf{97.7} & \textbf{90.5} & \textbf{3.19} & \textbf{96.8} & \textbf{95.2} \\ \midrule
POCO~\cite{boulch2022poco}          & \multirow{3}{*}{\begin{tabular}[c]{@{}c@{}}Synth.\\ Rooms\end{tabular}}   & 5.96          & 82.5          & 82.0          & 6.52          & 80.3          & 85.9          \\
NKF~\cite{williams2022neural}           &                                                                       & 9.15          & 66.5          & 83.4          &  9.87         & 69.3          & 86.2           \\
\textbf{Ours} &                                                                                                 & \textbf{5.38} & \textbf{86.6} & 86.4          & \textbf{5.01} & \textbf{90.5} & \textbf{91.8} \\ \midrule
\textbf{Ours} & CARLA                                                                                           & 3.20 & 95.9 & 89.1 & 3.08 & 98.1 & 95.0 \\ \midrule
\textbf{Ours} - \ks{} & Mixed                                                                                   & 3.72 & 93.6 & 89.1 & 3.17 & 97.4 & 95.5 \\ \bottomrule
\end{tabular}
\vspace{-2em}
\end{table}

%% file: tables/density.tex
\begin{table}[!t]
\centering
\small
\caption{\textbf{Performance comparison using different input densities.} Here the F-Score $\uparrow$ metric is shown.}
\label{tab:density}
\vspace{-1em}
\begin{tabular}{@{}l|cccccc@{}}
\toprule
Meters / scan    & 1.5           & 3             & 6             & 9             & 12            & 15            \\ \midrule
TSDF-Fusion~\cite{kazhdan2013screened} & 80.0          & 80.7          & 78.4          & 74.8          & 70.6          & 67.8          \\
SPSR~\cite{kazhdan2013screened}        & 88.2          & 88.3          & 87.9          & 86.4          & 83.3          & 79.5          \\ \midrule
\textbf{Ours}        & \textbf{96.1} & \textbf{96.0} & \textbf{95.4} & \textbf{94.1} & \textbf{92.6} & \textbf{92.0} \\ \bottomrule
\end{tabular}
\end{table}

%% file: figures/ablation.tex
\begin{figure}
  \begin{minipage}[t]{0.7\linewidth}
    \vspace{0pt} 
    \centering
    \includegraphics[width=\linewidth]{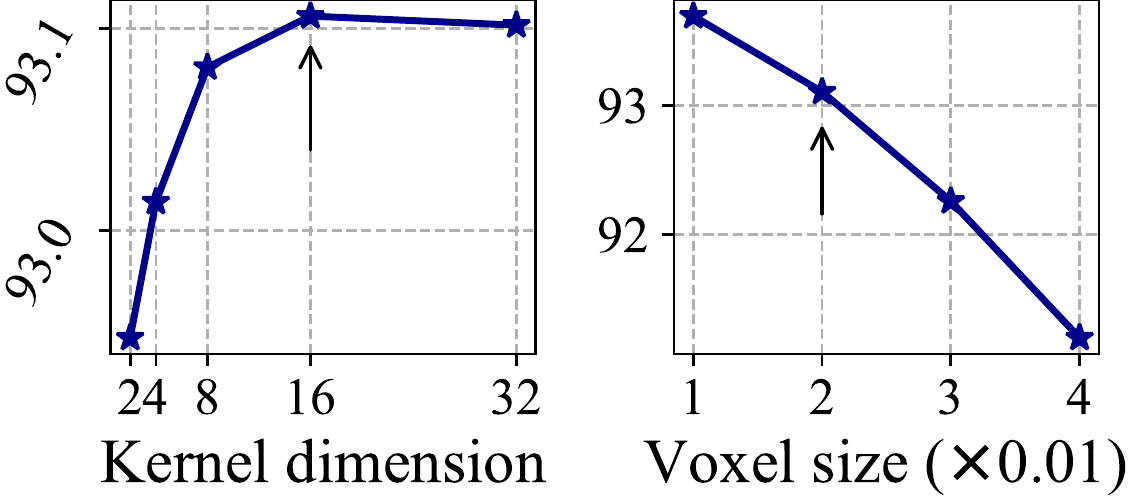}
  \end{minipage}%
  \begin{minipage}[t]{0.3\linewidth}
    \vspace{0.1em}
    \centering
    \small
\begin{tabular}{@{}cc@{}}
\toprule
\multicolumn{1}{l}{}                                & IoU                                                                          \\ \midrule
\begin{tabular}[c]{@{}c@{}}w/o\\ Hier.\end{tabular} & \begin{tabular}[c]{@{}c@{}}92.1\\ \textbf{\color[HTML]{CB0000}(-1.0)}\end{tabular}                        \\ \midrule
\begin{tabular}[c]{@{}c@{}}w/o\\ Grad.\end{tabular} & \begin{tabular}[c]{@{}c@{}}91.5\\ \textbf{\color[HTML]{CB0000}(-1.6)}\end{tabular} \\ \bottomrule
\end{tabular}
\end{minipage}
\caption{\textbf{Ablation study.} IoU metric is shown. The back arrows indicate the setting we use to obtain \cref{tab:shapenet}.}
\label{fig:ablation}
\vspace{-1.5em}
\end{figure}

%% file: sections/sec5-conclusion.tex
\vspace{-0.5em}
\section{Conclusion}
\vspace{-0.5em}

In this paper we present \MethodName{}, an accurate and scalable surface reconstruction algorithm using the neural kernel field representation.
We show by extensive experiments that our method reaches state-of-the-art quality and efficiency, while enjoying good generalization to unseen data. We believe our method further pushes the boundary of the field of 3D reconstruction and makes deep-learning-based surface reconstruction more practical for general use.
For future work we will try further improving the reconstruction quality using more expressive kernel models, as well as reducing memory footprint to allow for even larger-scale reconstructions.


%% file: sections/sup1-method.tex
In this appendix, we first provide more details of our method, including necessary network designs and derivations in \cref{supp:sec:method}.
Details related to experiments, including hyperparameters, metrics, and baselines are documented in \cref{supp:sec:setting}.
The design of our pipeline allows for different extensions applicable to various scenarios, and these extensions are in \cref{supp:sec:ext}.
Finally, more visualizations of our results are shown in \cref{supp:sec:vis} and the accompanying video clips.

\section{Detailed Method}
\label{supp:sec:method}

\subsection{Network Architecture}
\label{supp:subsec:network}

We use a customized version of a U-Net-like structure that operates fully over sparse voxels and outputs an adaptive hierarchy for kernel field computation.

\parahead{Point Encoder}
Given the input point cloud $\Xin$ and voxel size $W$ for the finest level, we first identify the set of points that resides within each voxel.
For each voxel, we then run a residual PointNet\cite{qi2017pointnet}-like encoder network to pool all the points within it into a feature vector.
The network is illustrated in \cref{fig:supp:encoder}.
To allow for translational equivariance, we convert from the global coordinates of the input points to local coordinates within the voxels as input $\tilde{\Xin}$.
For most of the datasets demonstrated in the main paper, per-point normal $\Nin$ is required as an additional piece of information to disambiguate the orientations.
This information does not have to be very accurate and usually can be easily obtained from sensor positions $\Oin$.
$\tilde{\Xin}$ and $\Nin$ are concatenated as a 6-dimensional input that is fed into the point encoder.

\begin{figure}[t]
    \centering
    \includegraphics[width=\linewidth]{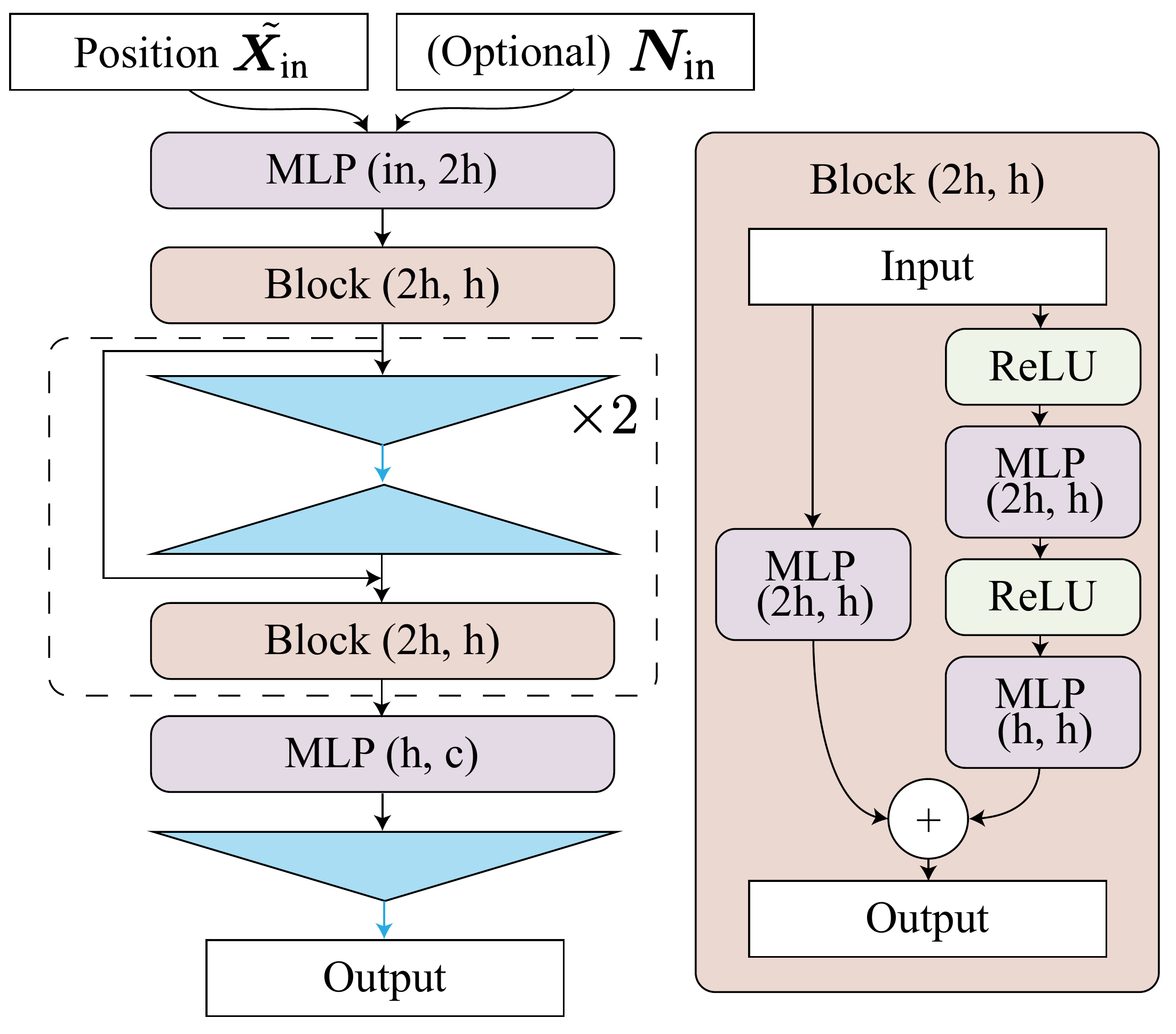}
    \caption{\textbf{Architecture for the point encoder.} Rounded rectangles show operations over each point, where $h$ is the hidden dimension. Blue triangles denote max pooling and repeating operations.}
    \label{fig:supp:encoder}
\end{figure}

\parahead{U-Net Encoder}
After quantizing per-point information into voxel-level features, we obtain a sparse voxel grid.
We then apply a sequential sparse convolution layers~\cite{graham2017submanifold} sandwiched by max pooling layers to coarsen the voxels, as shown in the upper part of \cref{fig:supp:unet}.
Deeper layers have larger receptive field and conceptually cover the area of $2^{l-1}W$.

\begin{figure}
    \centering
    \includegraphics[width=\linewidth]{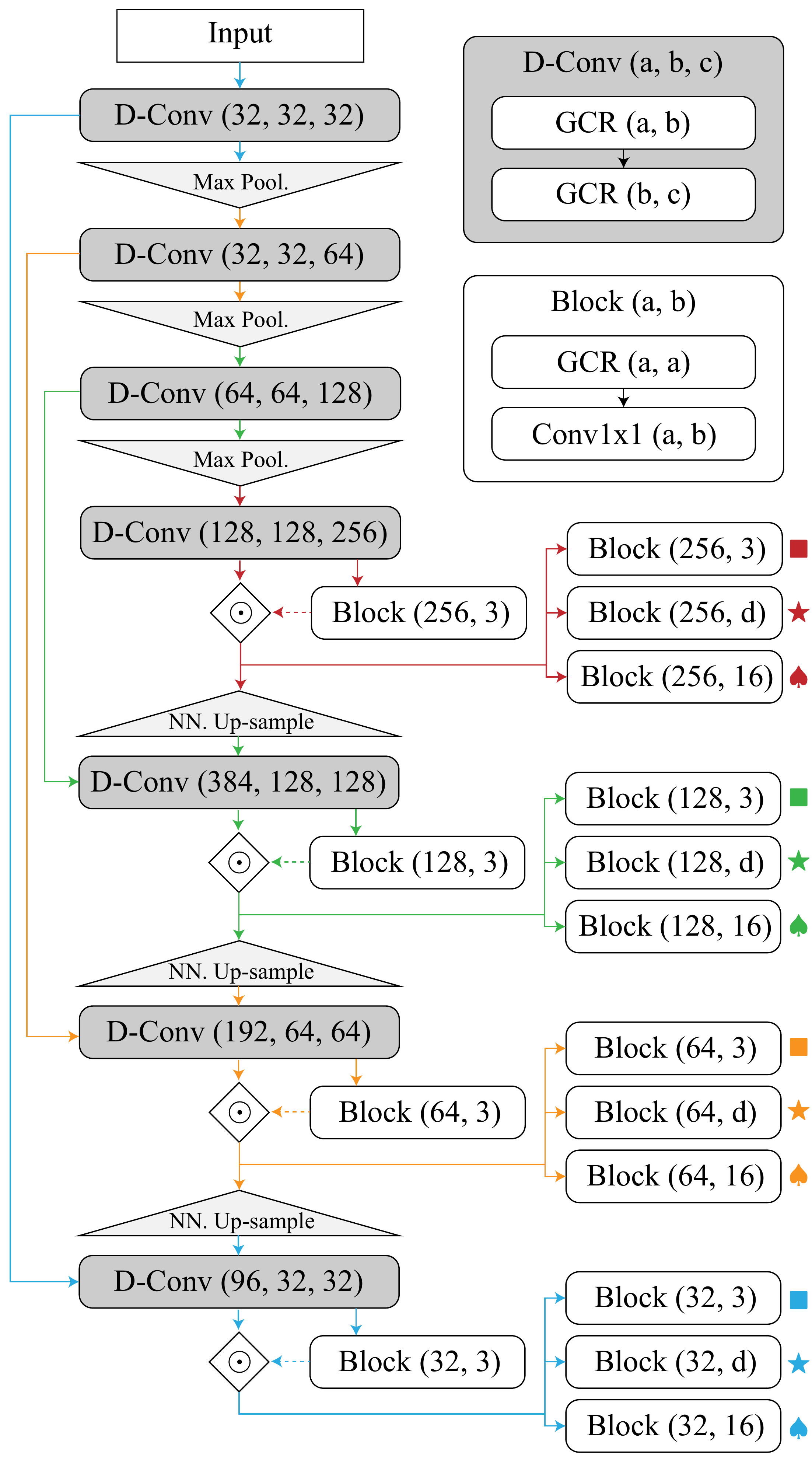}
    \caption{\textbf{Architecture for the U-Net.} `GCR' means a sequential application of GroupNorm, Convolution and ReLU activation. `$\odot$' denote voxel masking using the structure prediction results. Different colors of the arrows represent features at different layers in the hierarchy.}
    \label{fig:supp:unet}
\end{figure}

\parahead{U-Net Decoder}
As a reverse process of encoding, the decoder of the sparse U-Net also consists of several convolution layers with nearest-neighbour-based up-sampling, as illustrated in the lower part of \cref{fig:supp:unet}.
Skip connections are added to encourage fusion of low-level and high-level information.
For each layer we append additional task-specific branches to the backbone features that regress the following attributes needed in the following procedure:
\begin{itemize}[itemsep=0pt]
    \item Structure prediction branch outputs 3-dimensional features to determine the structure of the output hierarchy. Details are presented in the next paragraph.
    \item Normal prediction branch ($\blacksquare$) outputs the normals $\bm n_{i}^{(l)}$ that is 3-dimenensional and used later in the linear system. Note that there is no direct supervision to this branch and we find such a strategy provide better results due to the additional degrees of freedom introduced.
    \item Kernel prediction branch ($\bigstar$) outputs the features $\bm z_{i}^{(l)} \in \RR^d$ that is defined in the main text. MLP followed by Bézier interpolation are used to obtain the kernel field at arbitrary position.
    \item Mask prediction branch ($\spadesuit$) outputs 16-dimenensional features and are later transformed to a scalar value by MLP that determines whether the query position is far away from the real surface. In the main text the masking module is denoted as function $\varphi(\cdot)$.
\end{itemize}

\parahead{Structure Prediction}
We treat the 3-dimenensional features from the structure prediction branch as a 3-way classification score for each voxel.
Based on the classification, the voxels will be treated differently and altogether form a predicted new hierarchy, with the guarantee that the region defined by finer voxel is always covered by coarser voxels.
The semantics for these classifications are as follows: 
\begin{enumerate}
    \item \texttt{Subdivide}: the voxel should be subdivided into 8($=2^3$) voxels in the finer level.
    \item \texttt{Keep-as-is}: the voxel should be treated as a leaf node in the hierarchy, \ie, it is neither subdivided nor deleted.
    \item \texttt{Delete}: the voxel should be deleted from the hierarchy.
\end{enumerate}
An illustration for the structure prediction is shown in \cref{fig:supp:structure}.
Note that the hierarchy forms on the fly with the decoding process, and the other feature prediction branches are based only on the existing voxels (\ie, not classified as \texttt{Delete}).

\begin{figure}[t]
    \centering
    \includegraphics[width=\linewidth]{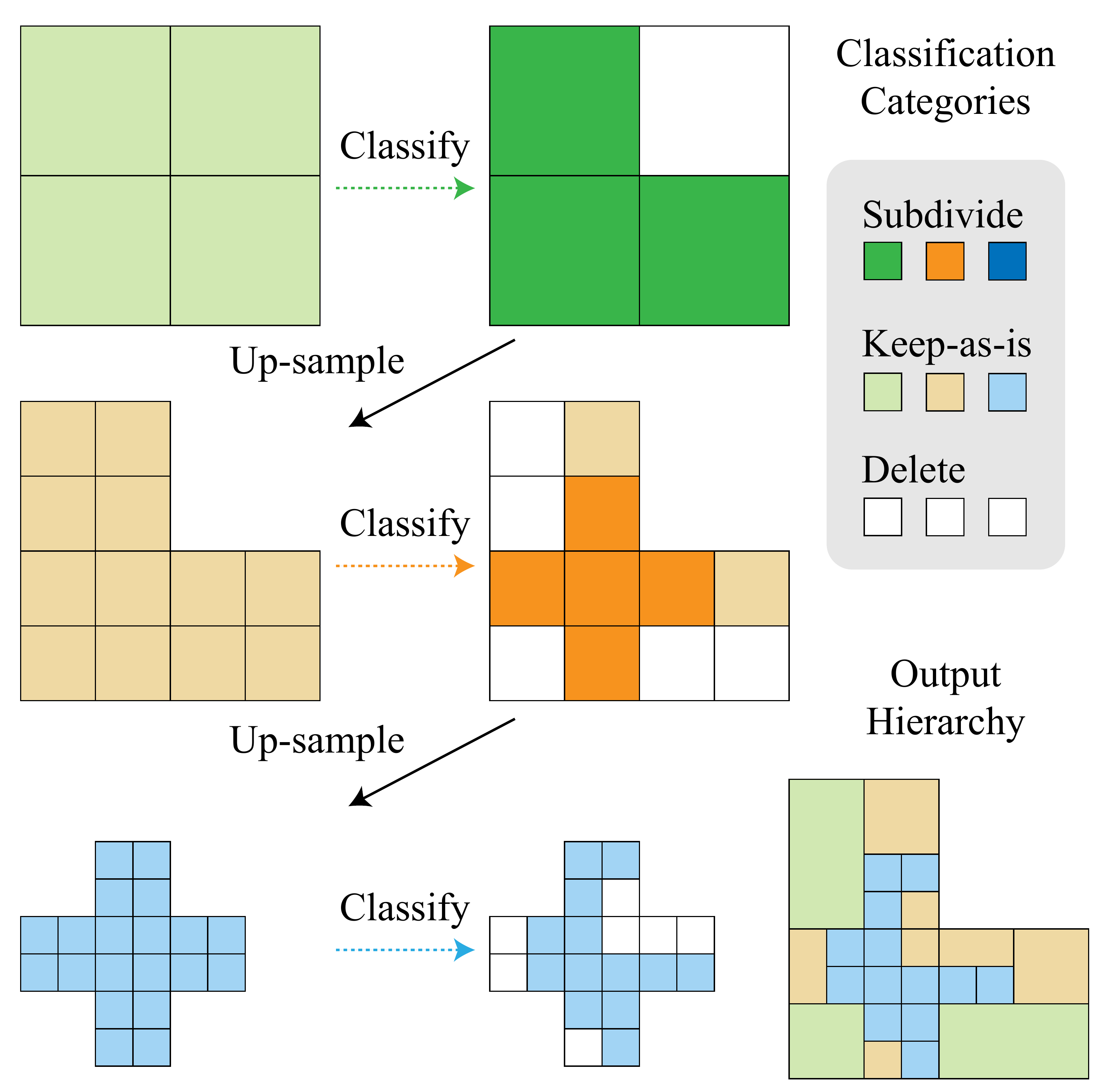}
    \caption{\textbf{Structure prediction.} We use a $L=3$ hierarchy as an example here. According to different classification results, the voxels will be treated differently.}
    \label{fig:supp:structure}
\end{figure}

\subsection{Hierarchical Kernel Formulation}
\label{supp:subsec:kernel}

We now give a detailed description of our hierachical neural kernel field formulation and procedure for solving for coefficients during inference. We then prove several facts about our formulation: namely that our learned kernel is indeed a kernel, that our predicted implicit belongs to an RKHS defined by that kernel, and that our linear system is symmetric and positive definite, and thus corresponds to a Gram matrix. 

\parahead{Defining the Neural Kernel Field} 
Recall that our shape is encoded as the zero level set of a Neural Kernel Field $f_\theta: \RR^3 \rightarrow \RR$ defined as a weighted combination of \emph{positive definite kernels} which are conditioned on the inputs and centered at the midpoints $\bm x_i^{(l)} \in \RR^3$ of each voxel in the predicted hierarchy:
\begin{align}\label{eq:supp_kernel}
    f_\theta(\bm x | \Xin, \Nin) = \sum_{i, l} \alpha_i^{(l)} K^{(l)}_\theta(\bm x, \bm x_i^{(l)} | \Xin, \Nin),
\end{align}
where $\alpha_i^{(l)} \in \RR$ are scalar coefficients at the $i^\text{th}$ voxel at level $l = 1, \ldots L$ in the hierarchy, and $K^{(l)}_\theta$ is the predicted kernel for the $l^\text{th}$ level defined as 
\begin{equation}
\begin{aligned}
    K^{(l)}_\theta(\bm x, \bm x') = \langle &\phi_\theta^{(l)}(\bm x; \Xin, \Nin),\\ &\phi_\theta^{(l)}(\bm x'; \Xin, \Nin) \rangle \cdot K^{(l)}_\text{b}(\bm x, \bm x').\nonumber
\end{aligned} 
\end{equation}
Here, $\langle\cdot,\cdot\rangle$ is the dot product, $\phi_\theta^{(l)} : \RR^3 \rightarrow \RR^d$ is the feature field extracted from the $l^\text{th}$ level of the hierarchy, and $K^{(l)}_b : \mathbb{R}^3 \times \mathbb{R}^3 \rightarrow \mathbb{R}$ is the \emph{Bézier Kernel}:
\begin{equation}
    K_b^{(l)}(\bm x, \bm x') = \psi^2(\frac{\bm x_x - \bm x'_x}{2^{l-1}W}) \cdot \psi^2(\frac{\bm x_y - \bm x'_y}{2^{l-1}W}) \cdot \psi^2(\frac{\bm x_z - \bm x'_z}{2^{l-1}W}),\nonumber
\end{equation}
with $\psi^2: \RR \rightarrow \RR$ the univariate second order B-spline:
\begin{equation}
    \psi^2(s) = \begin{cases}
    (s + \frac{3}{2})^2               & \text{if } s \in [-\frac{3}{2}, -\frac{1}{2}]\\
    -2s^2 + \frac{3}{2}               & \text{if } s \in [-\frac{1}{2}, \frac{1}{2}]\\
    (s - \frac{3}{2})^2               & \text{if } s \in [\frac{1}{2}, \frac{3}{2}]\\
    0                                 & \text{otherwise}
    \end{cases}\nonumber
\end{equation}
which decays to zero in a one-voxel (at level-$l$) neighborhood around its origin. 

\begin{lemma}\label{lem:posdef}
The basis functions~\eqref{eq:supp_kernel} used to construct our hierarchy are positive definite kernels.
\end{lemma}
\begin{proof}
The kernel $K^{(l)}_\theta(\bm x, \bm x')$ at each level is defined as the dot product between features $\phi_\theta^{(l)}(\bm x)$ and $\phi_\theta^{(l)}(\bm x')$ multiplied by the Bézier Kernel $K_b$. A kernel, by definition is a dot product of feature embeddings (Definition 2.8 in~\cite{shawe2004kernel}), and the product of kernels is a kernel (Proposition 3.22 in~\cite{shawe2004kernel}). Therefore each $K^{(l)}_\theta$ is a kernel.
\end{proof}


\begin{remark} 
Our functions $f_\theta$ defined on the hierarchy of kernels $K_\theta$ belong to an RKHS $\mathcal{H}$ induced by a kernel $\mathcal{K}$. This follows immediately from the Lemma~\ref{lem:posdef} and the Moore–Aronszajn theorem \cite{aronszajn1950theory}.
\end{remark}

\parahead{Computing a 3D Implicit Surface from Points} 
Recall that we compute an implicit surface by finding optimal coefficients $\bm \alpha^* = \{\{\alpha^{(l)}_i\}_{l=1}^L\}_{i=1}^{n^{(l)}}$ for the kernel field \eqref{eq:supp_kernel}. \ie, given the predicted voxel hierarchy, learned kernels $K_\theta^{(l)}$, and predicted normals $\bm n_j^{(l)}$, we minimize the following loss in the forward pass of our model (omitting the conditioning of $f_\theta$ on $\Xin, \Nin$ for brevity):
\begin{equation}\label{eq:lin_loss_supp}
\begin{aligned}
    \bm \alpha^* = \argmin_{\alpha_i^{(l)}} & \sum_{l=1}^{L'} \sum_{i = 1}^{n^{(l)}} \|\nabla_{\bm x} f_\theta(\bm x_i^{(l)}) - \bm n_i^{(l)}\|_2^2 + \\&\sum_{j = 1}^{n_\text{in}} |f_\theta(\bm x^\text{in}_j)|^2,
\end{aligned}
\end{equation}
where $L'\leq L$ is a hyper-parameter for the hierarchy.
We can rewrite~\eqref{eq:lin_loss_supp} in matrix form
\begin{equation}
    \argmin_{\bm \alpha} \|\Qmat \bm \alpha - \bm n\|^2_2 + \|\Gmat \bm \alpha \|_2^2,\nonumber
\end{equation}
where 
\begin{equation}
    \begin{aligned}
        \Gmat &= \begin{bmatrix} \Gmat^{(1)} & \ldots & \Gmat^{(L)} \end{bmatrix}, \;\; \Qmat = \begin{bmatrix} \Qmat^{(1)} & \ldots & \Qmat^{(L)} \end{bmatrix}, \\
        \bm n &= \begin{bmatrix} \bm n_{1, x} & \bm n_{1, y} & \bm n_{1, z} & \ldots & \bm n_{n_\text{v}, x} & \bm n_{n_\text{v}, y} & \bm n_{n_\text{v}, z} \end{bmatrix}^\top, \\
        \bm \alpha &= \begin{bmatrix} \alpha_1^{(1)} & \ldots & \alpha_{n^{(1)}}^{(1)} & \ldots & \alpha_1^{(L)} & \ldots & \alpha_{n^{(L)}}^{(L)} \end{bmatrix}^\top.
    \end{aligned}\nonumber
\end{equation}
Here $n_\text{v} = \sum_{l=1}^{L'} n^{(l)}$ and the matrix $\Gmat$ is the Gram matrix of the kernel defined as:
\begin{equation}
    \Gmat^{(l)}_{i, j} = K_\theta^{(l)}(\bm x^\text{in}_i, \bm x_j^{(l)}),\nonumber
\end{equation}
and the matrix $\Qmat$ is the matrix of partial derivatives of $\Gmat$ defined as:
\begin{equation}
        \Qmat = \begin{bmatrix} \Qmat^{(1)} & \ldots & \Qmat^{(L)} \end{bmatrix}, \quad
        \Qmat^{(l)} = \begin{bmatrix} \Qmat^{(l)}_x & \Qmat^{(l)}_y & \Qmat^{(l)}_z \end{bmatrix}\nonumber
\end{equation}
with
\begin{equation}
    \Qmat^{(l)}_{[x|y|z], i, j} = \partial_{[x|y|z]} K_\theta^{(l)}(\bm x^\text{in}_i, \bm x_j^{(l)}).\nonumber
\end{equation}
Setting the gradient with respect to $\bm \alpha$ of~\eqref{eq:lin_loss_supp} to $\bm 0$, we find that the optimal $\bm \alpha^*$ is the solution to the linear system:
\begin{equation}
    (\Qmat^\top \Qmat + \Gmat^\top \Gmat) \bm \alpha = \Qmat^\top \bm n.\nonumber
\end{equation}

\begin{lemma}
The matrix $\Qmat^\top \Qmat + \Gmat^\top \Gmat$ used to solve for the coefficients $\alpha^{(l)}_i$ is symmetric and positive definite.
\end{lemma}
\begin{proof}
The $n \times n$ matrix $\Qmat^\top \Qmat$ is symmetric and positive definite since $\forall \bm x \neq \bm 0, x^\top \Qmat^\top \Qmat x = \|\Qmat \bm x\|_2^2 \geq 0$. Furthermore since $\Qmat$ is constructed as a concatenation of Gram Matrices, it is full rank and thus $\|\Qmat \bm x\|_2^2 > \bm 0$. The same holds for $\Gmat^\top \Gmat$, and since the sum of positive definite matrices is positive definite, $\Qmat^\top \Qmat + \Gmat^\top \Gmat$ is positive definite.
\end{proof}


\begin{figure}[t]
    \centering
    \includegraphics[width=\linewidth]{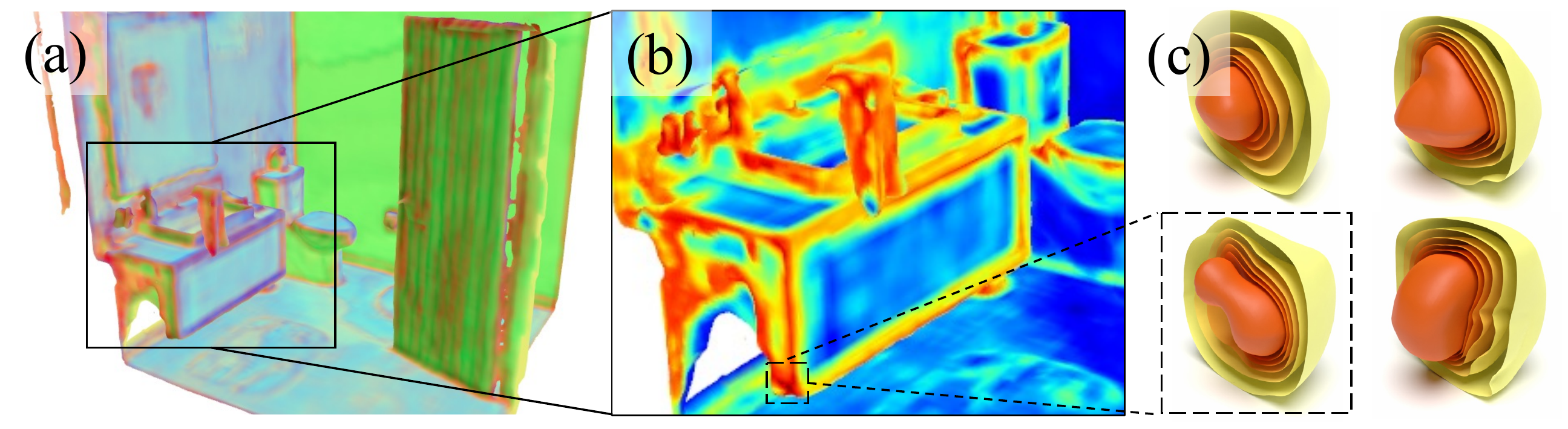}
    \caption{\textbf{Kernel visualization.} (a) PCA of the kernel features $\phi_\theta^{(l)}$. (b) Heatmap of kernel similarities \wrt one selected voxel in the dashed box. (c) Level sets of the kernel basis functions $K^{(l)}_\theta(\bm{x},\bm{x}_j^{(l)})$.}
    \label{fig:kernel}
\end{figure}

\parahead{Neural Kernel Visualization}
The above learned kernel formulation makes the notion of inductive bias precise. 
Solutions to the ridge regression minimize the learned RKHS norm $\|\cdot\|_\mathcal{H}$ which controls the behavior of the fitted surfaces away from the input points. 
This norm tightly controls the inductive bias of solutions and is meta-learned to perform well on the reconstruction task. 
In \cref{fig:kernel}, we perform PCA over the kernel features on the reconstructed surface and plot exemplar kernel basis functions $K^{(l)}_\theta(\bm{x},\bm{x}_j^{(l)})$ in 3D. Note how similar geometries share similar learned kernels shown in the heatmap. 

\subsection{Additional Losses}

\parahead{Structure Loss}
We compute a structure prediction loss on the predicted voxel hierarchy, written as:
\begin{equation}
    \mathcal{L}_\text{struct} = \sum_{i,l} \text{Cross-Entropy} \left(\bm c_i^{(l)}, (\bm c_i^{(l)})_\text{GT}\right),\nonumber
\end{equation}
where $\bm c_i^{(l)} \in \RR^3$ refers to the output of the structure prediction branch, and $(\bm c_i^{(l)})_\text{GT}$ is its ground-truth counterpart.
To compute the ground-truth hierarchy, we apply the approach in OctField~\cite{tang2021octfield} to $\Xdense$ and $\Ndense$.
Specifically, we start by building a dense hierarchy of the coarest level of voxels.
Then we recursively subdivide a voxel (suppose the volume it takes is $R_i^{(l)}$) into 8 voxels when the following criterion is satisfied:
\begin{equation}
    \mathbb{E}_{(\x, \bm n) \in R_i^{(l)}} \left( \text{std.}(\bm n_x) + \text{std.}(\bm n_y) + \text{std.}(\bm n_z) \right) > 0.1,\nonumber
\end{equation}
where $\x \in \Xdense$ and $\text{std.}(\cdot)$ stands for standard deviation.
Notably, we introduce another parameter $L'$ for the hierarchy denoting the maximum adaptive depth.
We run a second pass through the hierarchy to make sure that none of the voxels with depth $l>L'$ is a leaf node.

\parahead{Masking Loss}
To supervise $\varphi(\x)$ for trimming spurious geometry from shapes with open surfaces, we apply a binary-cross-entropy loss, ensuring that points which are within the distance $W$ from any point in $\Xdense$ are $1$ and $0$ otherwise.

\subsection{Out-of-Core Reconstruction}
When \MethodName{} is applied to very large scenes with millions of points, the $\Gmat$ and $\Qmat$ matrices become inevitably huge and could hardly fit into the GPU memory of a single video card.
Hence, we opt to divide the large scenes into several chunks with overlap, run our full pipeline on each of the chunks and then merge the reconstructions in its implicit form.
Due to the energy minimization nature of our algorithm, the overlapping regions of different chunks share the same constraints and are hence highly coherent.
For outdoor scenes with open surfaces, we merge the implicit functions in a way that also considers the output of the masking module, as illustrated in \cref{fig:supp:merge}.

\begin{figure}[t]
    \centering
    \includegraphics[width=\linewidth]{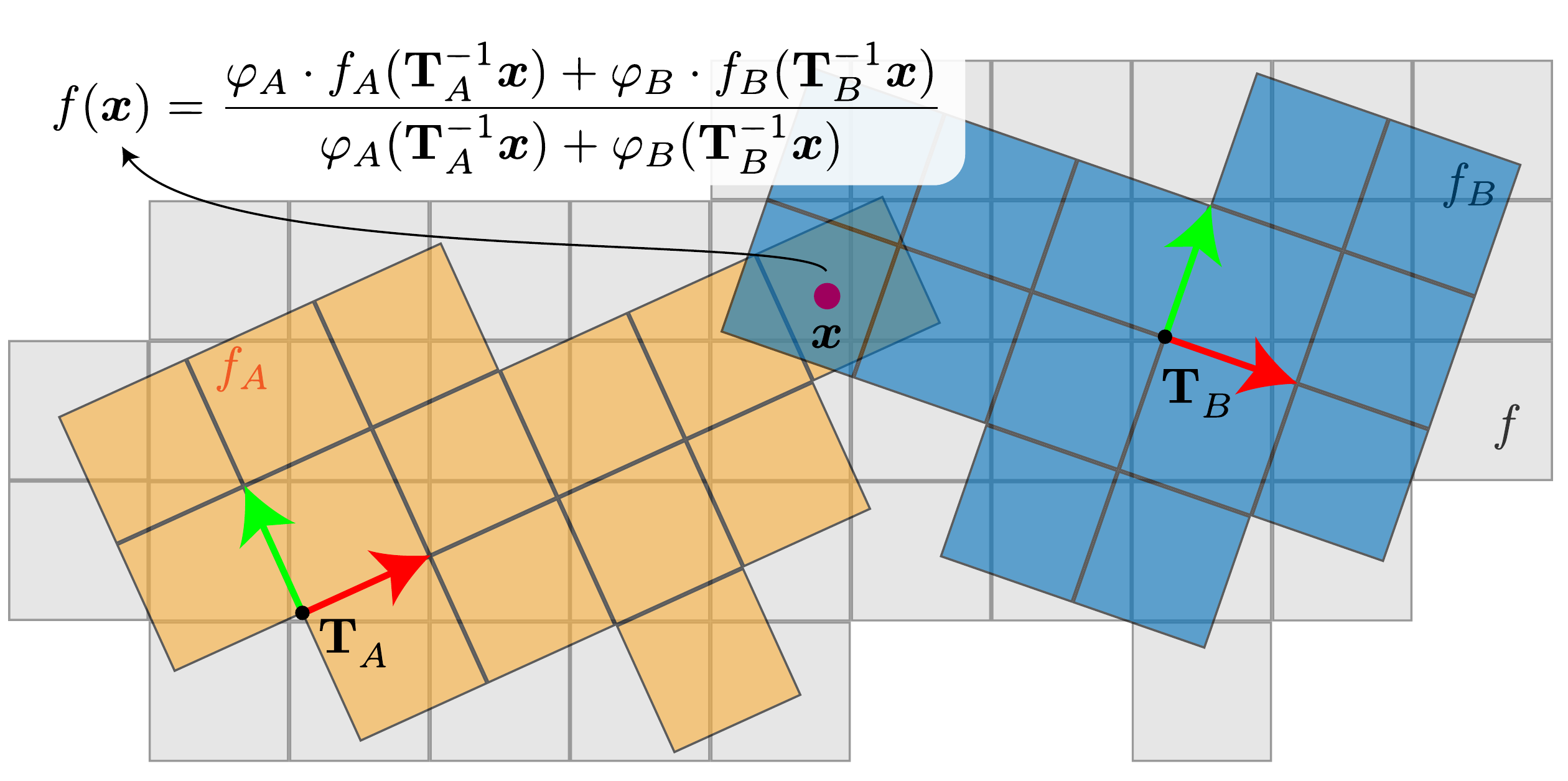}
    \caption{\textbf{Merging multiple reconstructions.} We demonstrate the merging operation with two chunks $A$ and $B$. The final implicit value for point $\x$ is defined as the average of the two chunks, weighted by their predicted masking values.}
    \label{fig:supp:merge}
\end{figure}

Mathematically, the final merged implicit field $f$ and the masking function $\varphi$ are defined as:
\begin{equation}
\begin{aligned}
    f(\x) &= \sum_k \varphi_k(\trans_k^{-1}\x) f_k(\trans_k^{-1}\x) / \sum_k \varphi_k(\trans_k^{-1}\x), \\
    \varphi(\x) &= \max_k \varphi_k(\trans_k^{-1}\x),
\end{aligned}\nonumber
\end{equation}
where the index $k$ refers to the chunks whose regions cover $\x$, and $\trans_k \in \mathbb{SE}(3)$ is the transformation of the chunk.
To extract the triangular mesh, we build a new hierarchy encapsulating all the hierarchies of the chunks and run Dual Marching Cubes~\cite{schaefer2004dual} over it.

%% file: sections/sup2-setting.tex
\section{Experimental Settings}
\label{supp:sec:setting}

\subsection{Hyperparameters}

\parahead{Shared Parameters}
To train the model we adopt a batch size of 4 using the technique of gradient accumulation.
We use the Adam optimizer with an initial learning rate of $10^{-4}$, and decay it to 70\% every 50K iterations. 
The gradients are clipped with a norm threshold of 0.5 to protect the model under spurious gradients.
To train the structure prediction branch along with other data branches, we use a warm-up strategy for the structures, where we start with the ground-truth structure and gradually increase the probability that the ground-truth structure is replaced with the predicted structure. 
We find this stabilizes training.
We use the Jacobi-preconditioned Conjugate Gradient solver for both the forward and the backward passes of the linear solve, and set the convergence tolerance to $10^{-5}$.
The solver typically converges within several hundreds of iterations.

\parahead{Dataset-Specific Parameters}
Due to the different scales and attributes of the datasets we tested on, we empirically choose different parameters for them, as listed in \cref{tab:supp:parameter}.
Notably, in the \emph{kitchen-sink-model} (\ks{}), we normalize all the training data to align with the scale of CARLA dataset during both training and testing.
After normalization, the average number of points per voxel is around 5.

\begin{table}[!t]
\setlength{\tabcolsep}{5.pt}
\centering
\small
\caption{\textbf{Dataset-specific hyperparameters.}}
\label{tab:supp:parameter}
\vspace{-1em}
\begin{tabular}{@{}lcccc}
\toprule
                    & ShapeNet & ABC & Room & CARLA \\ \midrule
Scale               & $1.1^3$  & $\sim1^3$ & $1^3$  & $51.2\text{m}^2$  \\
Voxel size $W$      & 0.02     &  0.02   & 0.01  &0.1       \\
Adaptive depth $L'$ & 1        & 2    & 2    & 2      \\
Kernel dim. $d$     & 16       & 4   & 4    &  4    \\ \bottomrule
\end{tabular}
\vspace{-1em}
\end{table}

\subsection{Baselines}

\parahead{SPSR~\cite{kazhdan2013screened}} 
We use the code from \url{https://github.com/mkazhdan/PoissonRecon}, and sets the voxel size (\texttt{width} parameter) to be the same as ours during comparison.
For trimming we use the density values provided along with the mesh. We remove vertices with densities lower than a given quantile which we determine empirically for each dataset.

\parahead{POCO~\cite{boulch2022poco}}
We use the official implementation from \url{https://github.com/valeoai/POCO}.
We tried our best to train a model with normal input (using the \texttt{normals} switch) but could not get a decently-performing model. \ie, the quality of the generated meshes are consistently much worse than the version without normal input.
Hence, for datasets where they do not provide a pretrained model, we train from scratch using our data without normals.

\parahead{NGSolver~\cite{huang2022neuralgalerkin}}%
We use the official implementation from \url{https://github.com/huangjh-pub/neural-galerkin}, taking the default configurations provided by the repository.

\parahead{SAP~\cite{peng2021shape} and ConvONet~\cite{peng2020convoccnet}}
We use the implementation from \url{https://github.com/autonomousvision/convolutional_occupancy_networks} and \url{https://github.com/autonomousvision/shape_as_points} respectively and take the official configurations whenever possible.
For comparisons with normal input, we modify their point encoder to accept an additional input of normal information through concatenation, similar to ours as in \cref{supp:subsec:network}.

\parahead{NKF~\cite{williams2022neural}}
We ask the original authors of the paper who kindly run all the comparisons for us because their code is not yet publicly available.

\parahead{IMLSNet~\cite{liu2021deep}}
The implementation is taken from \url{https://github.com/Andy97/DeepMLS} and we use the default configurations to re-train their network for settings where pretrained models are not available.

\parahead{TSDF-Fusion~\cite{vizzo2022sensors}}
We choose to use the implementation from \url{https://github.com/PRBonn/vdbfusion} among all others due to its efficiency.
As the algorithm requires sensor rays instead of points and normals, we generate pseudo-rays emitting from $\x + \epsilon\bm{n}$ and stopping at $\x$ as the input to their algorithm.

\parahead{LIG~\cite{huang2021di}}
We use the implementation from \url{https://github.com/huangjh-pub/di-fusion} with a pretrained local implicit auto-encoder that takes normal input.
Nearby local grids are blended with trilinear weights to ensure a smooth reconstruction.

\subsection{Metrics}

To compute the metrics, we densely sample points and the corresponding normals from both the ground-truth mesh (denoted as $\Xgt$ and $\Ngt$) and the predicted mesh (denoted as $\Xpd$ and $\Npd$).

\parahead{Chamfer Distance $d_C$}
The Chamfer distance is computed using:
\begin{equation}
\begin{aligned}
    d_C &= \frac{1}{2} (\text{Comp.} + \text{Acc.}), \\
    \text{Comp.} &= \frac{1}{|\Xgt|} \sum_{\xgt\in\Xgt} \min_{\xpd \in \Xpd} \|\xgt-\xpd\|, \\
    \text{Acc.} &= \frac{1}{|\Xpd|} \sum_{\xpd\in\Xpd} \min_{\xgt \in \Xgt} \|\xpd-\xgt\|.
\end{aligned}\nonumber
\end{equation}
Note that this is consistent with the one used in, \eg, ConvONet~\cite{peng2020convoccnet} but different with the one used in POCO\footnote{\url{https://github.com/ErlerPhilipp/points2surf/issues/20}}, hence the difference in the results.

\parahead{Normal Consistency}
The normal consistency score is defined as follows:
\begin{equation}
    \frac{1}{2} \left(\sum_{\xgt\in\Xgt} |\langle\ngt,\bm{n}_{\text{NN}(\xgt, \Xpd)}\rangle| + \sum_{\xpd\in\Xpd} |\langle\npd,\bm{n}_{\text{NN}(\x, \Xgt)}\rangle|\right).\nonumber
\end{equation}

\parahead{F-Score}
The F-Score is defined as follows:
\begin{equation}
    \frac{2 \cdot \text{Precision} \cdot \text{Recall}}{\text{Precision} + \text{Recall}},\nonumber
\end{equation}
where
\begin{equation}
\begin{aligned}
    \text{Precision} & = \frac{|\{ \xpd \in \Xpd \,|\, \min_{\xgt\in\Xgt} \|\xgt-\xpd\| < \xi \}|}{|\Xpd|} , \\
    \text{Recall} & = \frac{|\{ \xgt \in \Xgt \,|\, \min_{\xpd\in\Xpd} \|\xpd-\xgt\| < \xi \}|}{|\Xgt|}.
\end{aligned}\nonumber
\end{equation}
We use $\xi=0.01$ for object-level and indoor datasets, and $\xi=0.1$ for CARLA dataset.

\subsection{Details on CARLA Dataset}

We report detailed specifications of our generated CARLA dataset in \cref{tab:supp:carla}.
To obtain the input and ground-truth training pairs, we use a simulated LiDAR sensor that is mounted 1.8m above the ground, with a vertical field-of-view ranging from $-15^\circ$ to $15^\circ$ and an atmosphere attenuation rate of $4 \times 10^{-3}$.

\begin{table}[!t]
\centering
\small
\caption{\textbf{Dataset specifications for CARLA.}}
\label{tab:supp:carla}
\vspace{-1em}
\begin{tabular}{@{}lcccc}
\toprule
                & Town1 & Town2 & Town3 & Town10 \\ \midrule
Subset          & Original  & Original      & Novel      & Original       \\
\# Drives        & 3    &  3     &   3    &   4     \\
\# Chunks        & 93      & 93  & 90  &  124   \\
\# Avg. Points   &  510K     & 649K   & 546K & 388K    \\ \bottomrule
\end{tabular}
\vspace{-1em}
\end{table}

%% file: sections/sup3-ext.tex
\section{Extensions}
\label{supp:sec:ext}

\begin{figure*}
    \centering
    \includegraphics[width=\linewidth]{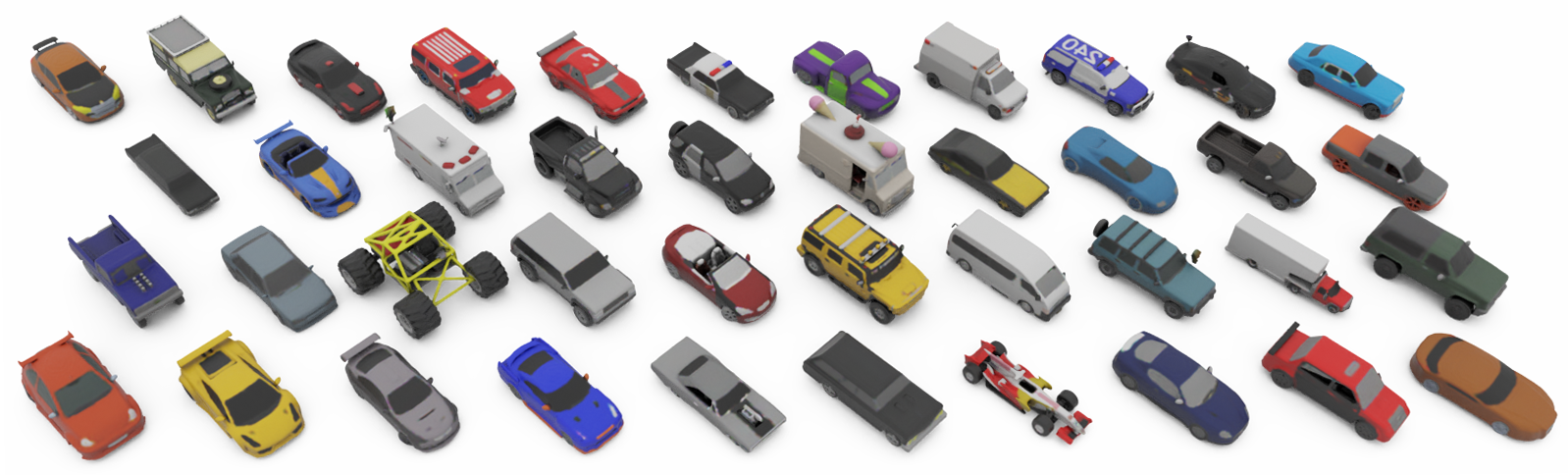}
    \caption{\textbf{Texture reconstruction.} Our sparse neural kernel hierarchy is expressive enough to faithfully represent the textures on the shape. We use 10K colored points sampled from ShapeNet~\cite{chang2015shapenet} cars as input and iterate 800 times for each shape.}
    \label{fig:supp:ext-color}
\end{figure*}

\subsection{Texture Reconstruction}

Our sparse neural kernel field representation defined over the hierarchy can be easily extended to represent other scene attributes, such as textures.
The textures recovered can be defined continuously in the region covered by the hierarchy, similar to TextureField~\cite{oechsle2019texture}.
Specifically, we define 3 additional implicit functions $g_\phi^R, g_\phi^G, g_\phi^B$ for the red, green, and blue channel of the texture field as:
\begin{equation}
    g_\phi^{[R|G|B]} (\x) = \sum_{i,l} \gamma_i^{(l),[R|G|B]} K_\phi^{(l),[R|G|B]} (\x, \x_i^{(l)}),\nonumber
\end{equation}
and the coefficients $\gamma_i^{(l)}$ can be obtained by solving the following linear system (using similar derivations as in \cref{supp:subsec:kernel}, omitting $R,G,B$ superscripts for brevity):
\begin{equation}
    \Gmat^\top_c \Gmat_c \bm{\gamma} = \Gmat^\top_c \bm t,
\end{equation}
where $\Gmat_c$ is the Gram matrix for the kernel $K_\phi$ and $\bm{t}$ is the input color vector.

To demonstrate our ability of texture reconstruction, we add 3 additional branches to our network backbone that predict kernel fields $K_\phi$ for the red, green and blue channel respectively.
We overfit some examplar cars from ShapeNet~\cite{chang2015shapenet} dataset with 10K colored input points and the results are shown in \cref{fig:supp:ext-color}.
We could accurately recover the textures along with the shape, showing a strong representation power for signals other than geomtry.

\begin{figure}
    \centering
    \includegraphics[width=\linewidth]{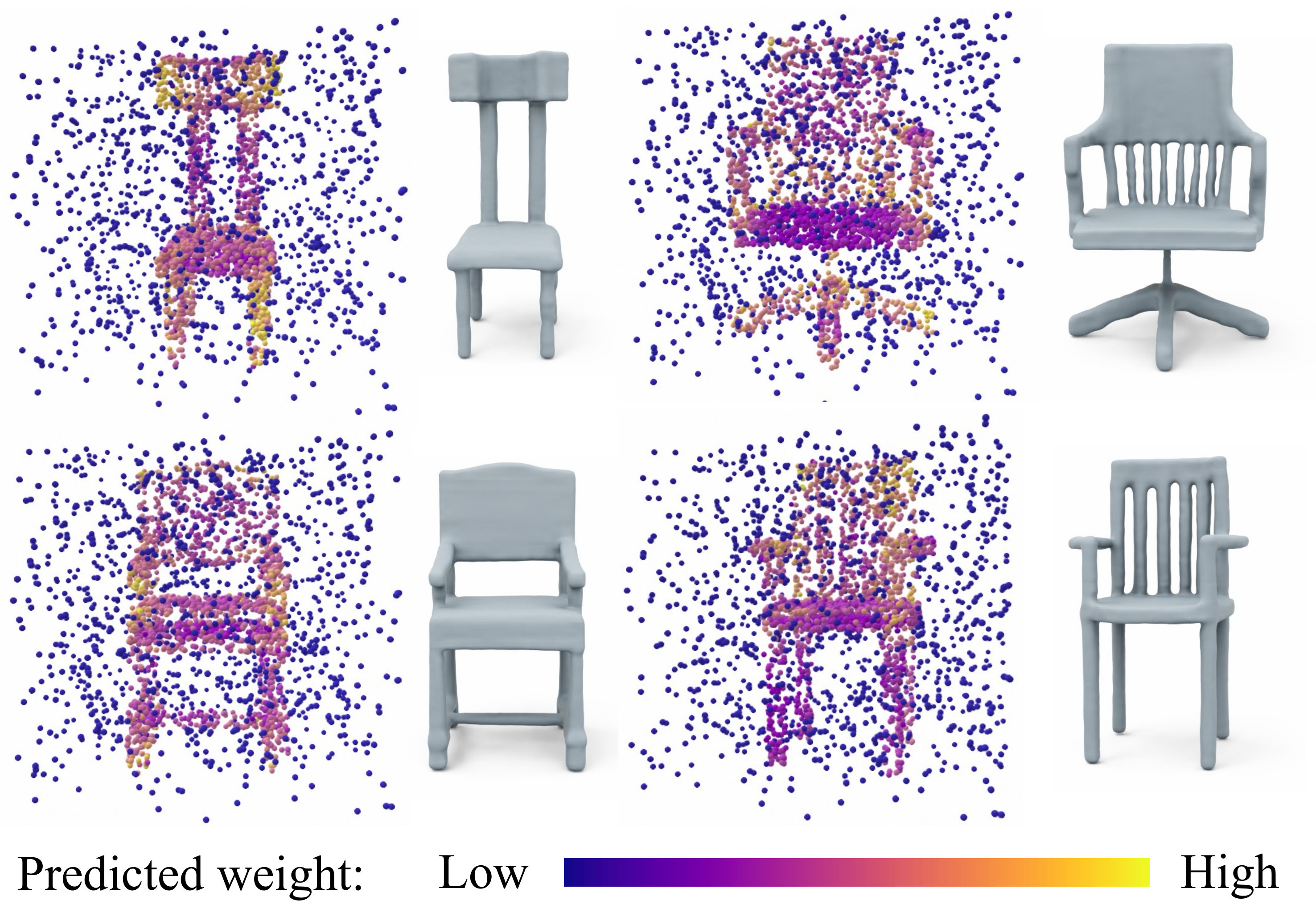}
    \caption{\textbf{Outlier detection and removal.} Given input points with extreme outliers (left), the model learns to automatically down-weigh irrelevant points and reconstructs good geometry (right).}
    \label{fig:supp:ext-weight}
\end{figure}

\subsection{Outlier Detection}

For input point clouds that are corrupted with outliers, the structure prediction branch can already prune many of them by not generating supporting voxels for regions that are faraway from the real surfaces.
However, for outliers that are close to the surface, they act as false data constraints which should not be included in our linear system.
To this end, we introduce a weighted version of our energy formulation \eqref{eq:lin_loss_supp} as follows:
\begin{equation}
\begin{aligned}
    \bm \alpha^* = \argmin_{\alpha_i^{(l)}} & \sum_{l=1}^{L'} \sum_{i = 1}^{n^{(l)}} \|\nabla_{\bm x} f_\theta(\bm x_i^{(l)}) - \bm n_i^{(l)}\|_2^2 + \\&\sum_{j = 1}^{n_\text{in}} \textcolor{red}{w_j^\text{in}} |f_\theta(\bm x^\text{in}_j)|^2,
\end{aligned}\nonumber
\end{equation}
where the highlighted variable $w_j^\text{in} \in [0,1]$ is defined for each input point and predicted by an MLP (ended with Sigmoid) that relies on the trilinearly-interpolated backbone features of our U-Net.

The change in the energy formulation only requires a minor change in the linear system as:
\begin{equation}
    (\Qmat^\top \Qmat + \Gmat^\top \textcolor{red}{\mathbf{W}} \Gmat) \bm \alpha = \Qmat^\top \bm n,\nonumber
\end{equation}
where $\mathbf{W} = \text{diag}(w_j^\text{in})$, and the gradients could also be propagated to the weights during training.

The model could then be trained without adding any extra supervision, and we show in \cref{fig:supp:ext-weight} that the model could automatically learn the weights of the points in a meaningful way, where the model is trained and tested on 3K-point input with 50\% of outliers.

%% file: sections/sup4-vis.tex
\section{More Visualizations}
\label{supp:sec:vis}

We provide more visualizations in \cref{fig:supp:abc}, \cref{fig:supp:snetw}, \cref{fig:supp:snetwo}, \cref{fig:supp:room}, \cref{fig:supp:carla1} and \cref{fig:supp:carla2}.

\begin{figure*}
    \centering
    \includegraphics[width=0.9\linewidth]{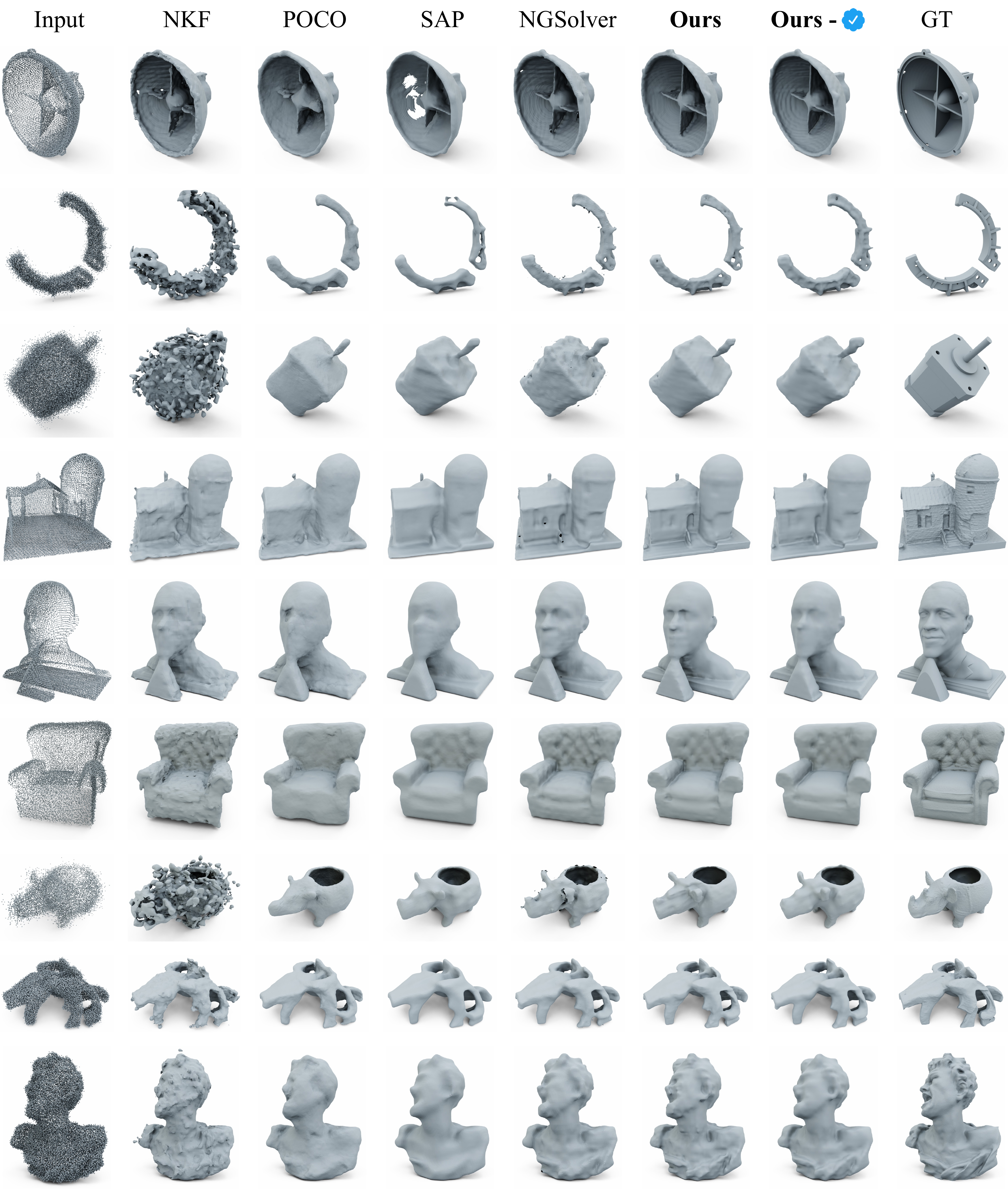}
    \caption{\textbf{More results on ABC/Thingi10K datasets.} Best viewed with 2$\times$ zoom-in.}
    \label{fig:supp:abc}
\end{figure*}

\begin{figure*}
    \centering
    \includegraphics[width=0.75\linewidth]{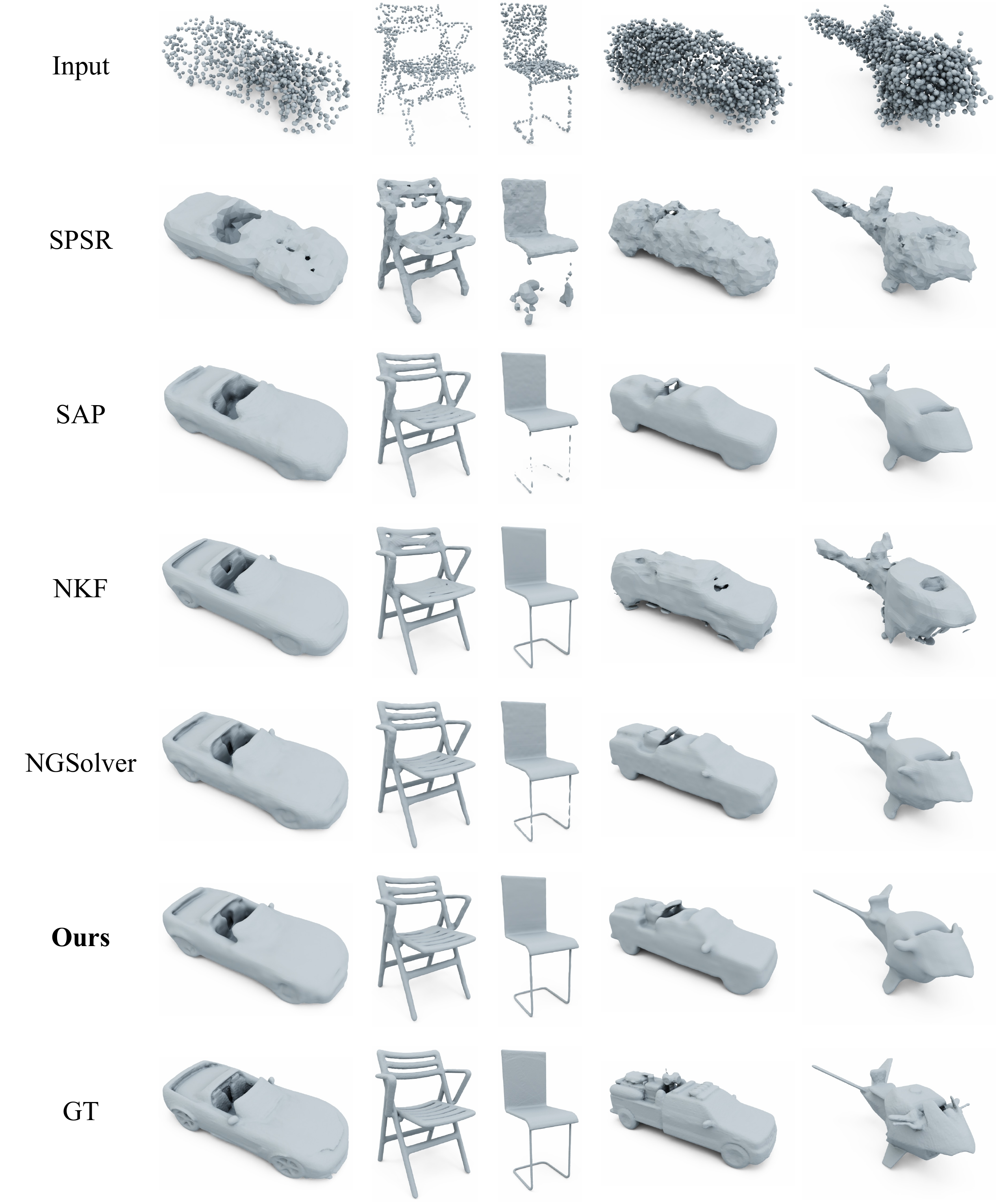}
    \caption{\textbf{More results on ShapeNet datasets (with normal).} Best viewed with 2$\times$ zoom-in.}
    \label{fig:supp:snetw}
\end{figure*}

\begin{figure*}
    \centering
    \includegraphics[width=0.9\linewidth]{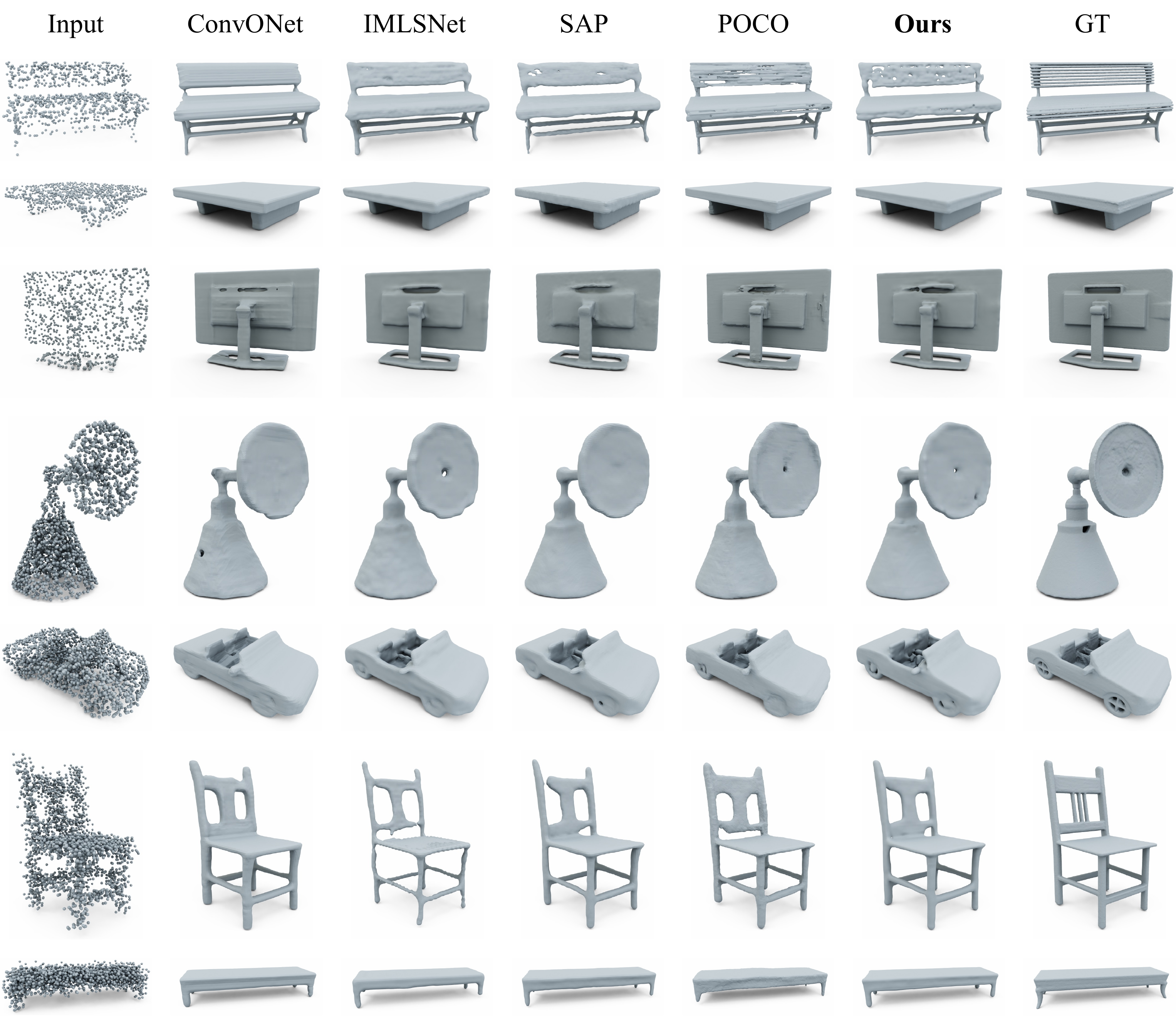}
    \caption{\textbf{More results on ShapeNet datasets (without normal).} Best viewed with 2$\times$ zoom-in.}
    \label{fig:supp:snetwo}
\end{figure*}

\begin{figure*}
    \centering
    \includegraphics[width=\linewidth]{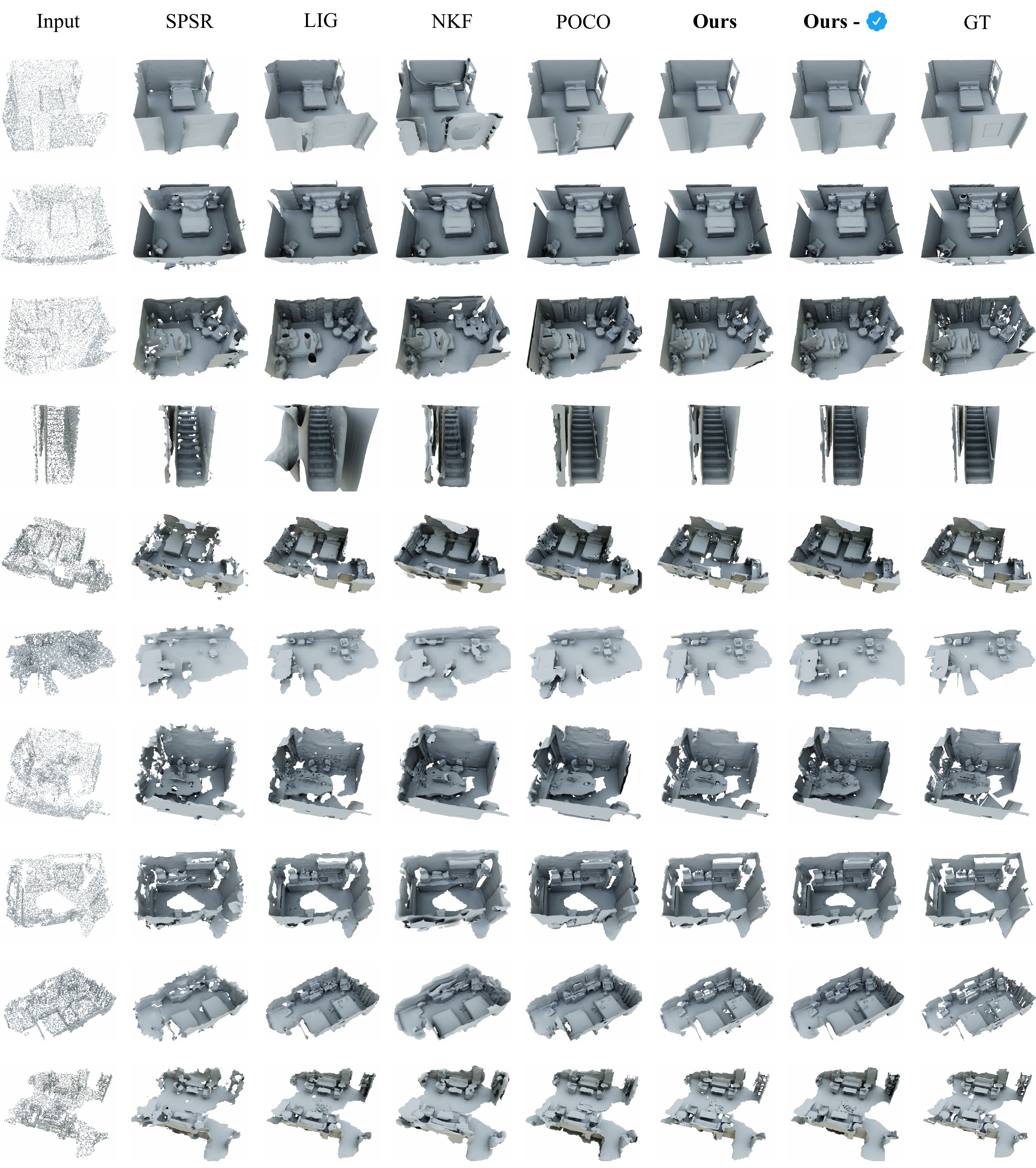}
    \caption{\textbf{More results on Matterport/ScanNet datasets.} Best viewed with 2$\times$ zoom-in.}
    \label{fig:supp:room}
\end{figure*}

\begin{figure*}
    \centering
    \includegraphics[width=\linewidth,page=1]{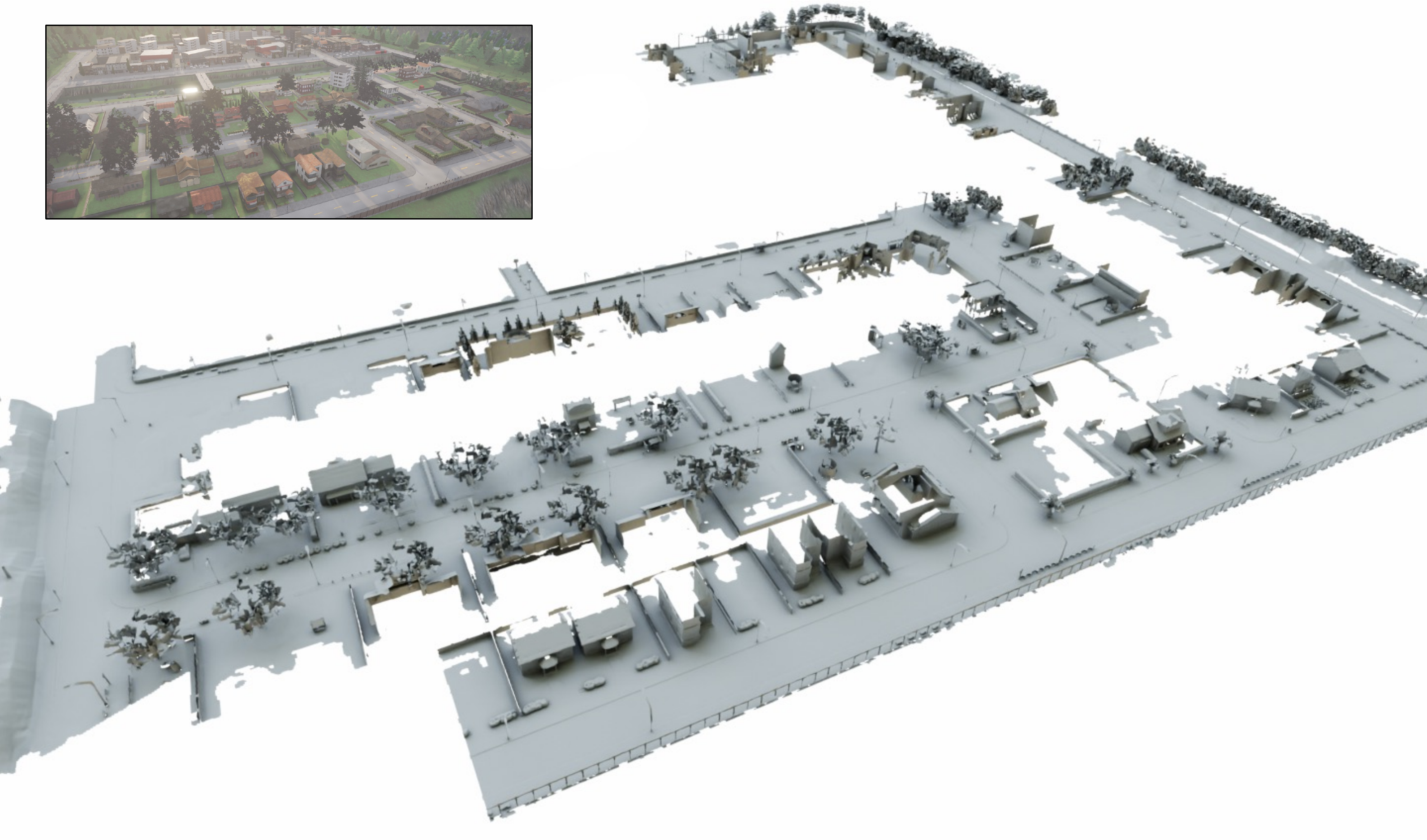}\\
    \includegraphics[width=\linewidth,page=2]{figures/img/supp-carla.pdf}
    \caption{\textbf{Our reconstruction of the CARLA dataset.} The inset shows RGB rendering of the scene within the simulator.}
    \label{fig:supp:carla1}
\end{figure*}

\begin{figure*}
    \centering
    \includegraphics[width=\linewidth,page=3]{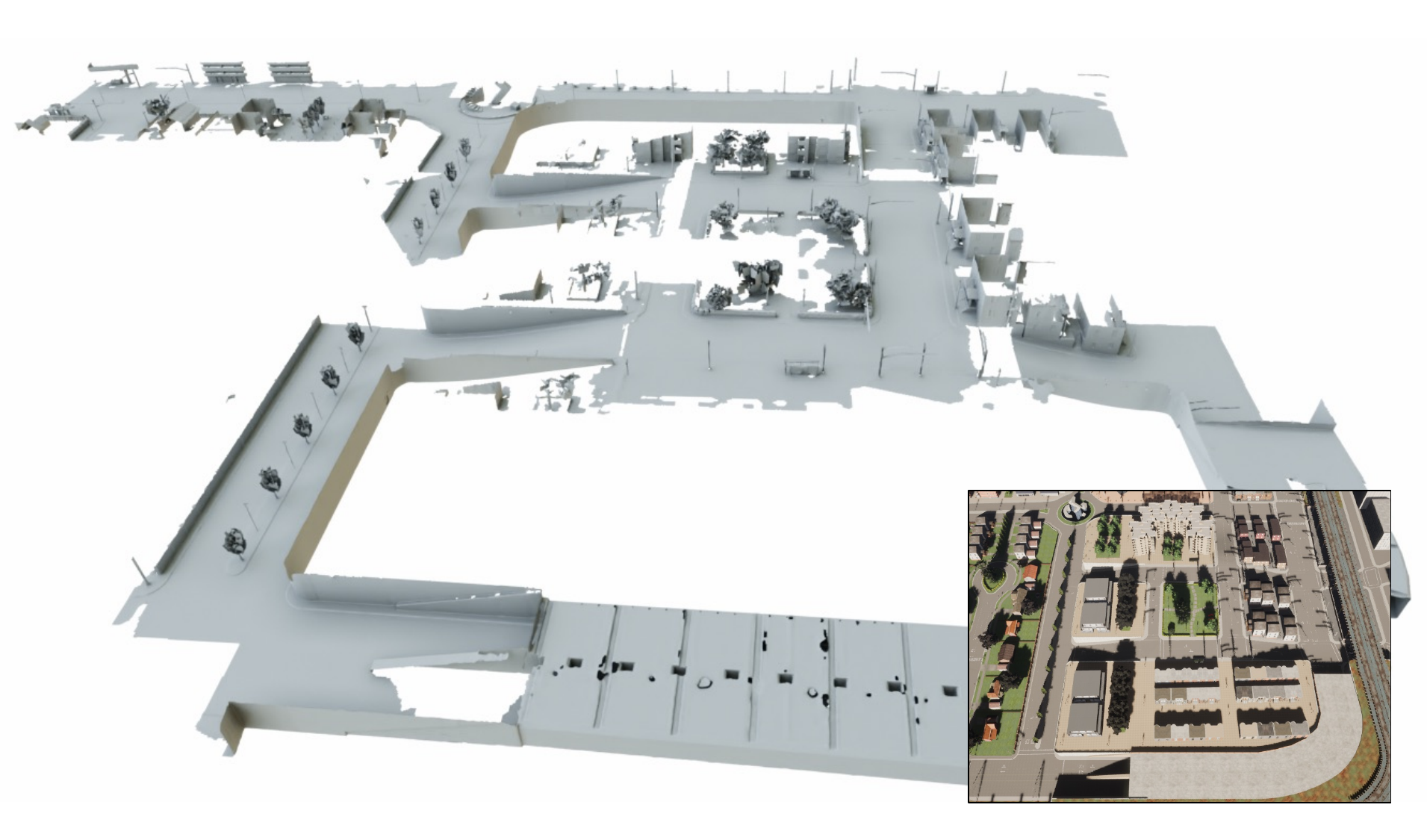}\\
    \includegraphics[width=\linewidth,page=4]{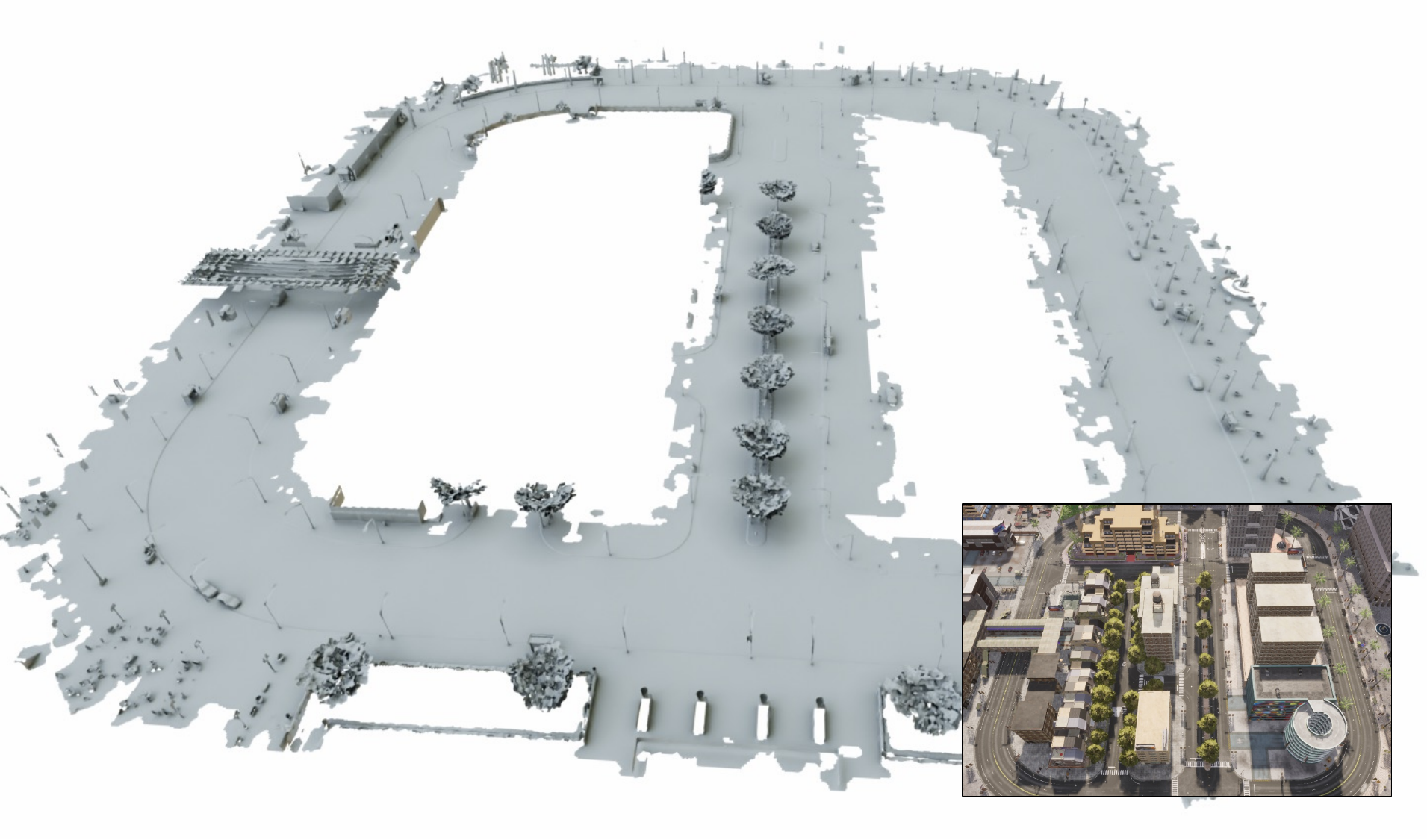}
    \caption{\textbf{Our reconstruction of the CARLA dataset.} The inset shows RGB rendering of the scene within the simulator.}
    \label{fig:supp:carla2}
\end{figure*}